%% file: main_arxiv.tex
\documentclass[11pt]{article}
\usepackage[utf8]{inputenc} 
\usepackage[T1]{fontenc}    

\input{command.tex}

\hypersetup{pdftitle={Online Inference for Quantile Distributional Reinforcement Learning},pdfauthor={Zijie Cheng,  Zhihua Zhang}}

\usepackage{enumerate}

\newtheorem{claim}{Claim}[section]
\newtheorem{lemma}[claim]{Lemma}
\newtheorem{assumption}{Assumption}

\newtheorem{theorem}{Theorem}[section]

\newtheorem{proposition}{Proposition}[section]
\newtheorem{remark}{Remark}[section]
\newtheorem{corollary}{Corollary}[section]

\usepackage{xr}

\title{Online Inference for Quantile Distributional Reinforcement Learning}

\author{
Zijie Cheng\thanks{School of Mathematical Sciences, Peking University; email: \texttt{xmwsbsj@stu.pku.edu.cn}.} \and
Zhihua Zhang\thanks{School of Mathematical Sciences, Peking University; email: \texttt{zhzhang@math.pku.edu.cn}.}
}

\begin{document}
\maketitle
\begin{abstract}
In this work, we initiate a systematic study of online statistical inference for quantile-based distributional policy evaluation, which aims to characterize the distribution of discounted cumulative rewards under a given policy. To obtain a finite-dimensional representation of the return distribution, we consider the fixed point $\bm\eta_m$ of the quantile-projected distributional Bellman equation, parameterized by the quantile vector $\btheta_m$. We develop two online inference procedures for $\btheta_m$, covering synchronous and asynchronous sampling from a generative model as well as Markovian data. The first method uses the trajectory of quantile temporal-difference learning (QTD). We establish functional central limit theorems for the partial sums of QTD iterates and construct a recursively computable random-scaling statistic. The resulting confidence intervals have asymptotically exact coverage at deterministic times, with length of order $T^{-1/2}$ in probability. The second method constructs anytime-valid confidence intervals directly from samples. We invert multiscale tests of the Bellman estimating equation and use dyadic reward histograms to evaluate conservative evidence bounds online. The resulting intervals have finite-sample simultaneous coverage, and the histogram contribution to their length bound is of order $T^{-\zeta}$ for $\zeta\in[1/2,1)$. Under positive limiting state-visit frequencies, the resulting confidence intervals have lengths of order $\sqrt{\log\log T/T}$, which we establish as classwise optimal through a matching lower bound. Together, these results provide complementary guarantees for online inference in quantile distributional reinforcement learning, combining asymptotically exact fixed-time inference with finite-sample anytime validity.
\end{abstract}
\section{Introduction}\label{Section:intro}
\input{intro}
\section{Preliminaries}\label{Section:preliminary}
\input{preliminary}
\section{Main Results}\label{Section:analysis}
\input{analysis}

\section{Proof Outlines}\label{Section:proof_outline}
\input{proof_outline}

\section{Numerical Experiments}\label{Section:numerical_experiments}
\input{numerical_experiments}

\section{Conclusions}\label{Section:conclusion}
\input{conclusion}

\bibliography{ref}
\bibliographystyle{abbrvnat}
\newpage

\appendix
\section{Proofs of Supporting Results in Section~\ref{Section:proof_outline}}\label{Appendix_omitted_proofs_4}
\input{Appendix_omitted_proofs_4}
\section{Proofs of Technical Lemmas}\label{Appendix_technical_lemmas}
\input{Appendix_technical_lemma}

\section{Additional Details on Numerical Experiments}\label{Appendix_numerical}
\input{Appendix_numerical}

\end{document}

%% file: command.tex
\usepackage{geometry}
\usepackage{setspace}
\usepackage{amsmath, amssymb, amsfonts, bm, mathtools, mathrsfs}
\usepackage{amsthm}
\usepackage[dvipsnames]{xcolor}
\definecolor{darkblue}{rgb}{0,0,.5}
\usepackage{graphicx}
\usepackage{subfigure}
\usepackage[numbers]{natbib}
\usepackage{url}            
\usepackage{booktabs}       
\usepackage{amsfonts}       
\usepackage{nicefrac}       
\usepackage{microtype}      

\allowdisplaybreaks

\usepackage{float}
\usepackage{multirow}
\usepackage{footnote}
\usepackage{dsfont}
\usepackage{mathabx}

\usepackage{algorithm}
\usepackage{algorithmic}
\usepackage{nicefrac}
\usepackage{tikz}
\usepackage{overpic}

\usepackage{dsfont}
\usepackage[colorlinks=true,allcolors=darkblue]{hyperref}
\usepackage[capitalize]{cleveref}
\usepackage{crossreftools}

\makeatletter
\newcommand*{\rom}[1]{\expandafter\@slowromancap\romannumeral #1@}
\makeatother

\newcommand{\ind}{\mathds{1}}

\newcommand{\brc}[1]{\left\{{#1}\right\}}
\newcommand{\prn}[1]{\left({#1}\right)} 
\newcommand{\brk}[1]{\left[{#1}\right]} 
\newcommand{\norm}[1]{\left\|{#1}\right\|} 
\newcommand{\abs}[1]{\left|{#1}\right|} 
\newcommand{\what}[1]{\widehat{#1}}
\newcommand{\wtilde}[1]{\widetilde{#1}}

\newcommand{\rd}{\mathrm{d}}
\newcommand{\normal}{\mathsf{N}}

\newcommand{\<}{\langle} 
\renewcommand{\>}{\rangle}

\newcommand{\simiid}{\stackrel{\textup{iid}}{\sim}}

\def\gA{{\mathcal{A}}}

\def\gC{{\mathcal{C}}}

\def\gF{{\mathcal{F}}}

\def\gH{{\mathcal{H}}}
\def\gI{{\mathcal{I}}}

\def\gM{{\mathcal{M}}}

\def\gO{{\mathcal{O}}}
\def\gP{{\mathcal{P}}}

\def\gS{{\mathcal{S}}}
\def\gT{{\mathcal{T}}}

\def\gV{{\mathcal{V}}}

\def\sP{{\mathscr{P}}}

\def\bA{{\bm{A}}}

\def\bB{{\bm{B}}}

\def\bD{{\bm{D}}}
\def\bh{{\bm{h}}}
\def\bH{{\bm{H}}}

\def\bw{{\bm{w}}}

\def\bx{{\bm{x}}}
\def\bI{{\bm{I}}}

\def\bG{{\bm{G}}}

\def\bU{{\bm{U}}}

\def\bq{{\bm{q}}}
\def\bPi{{\bm{\Pi}}}
\def\bDelta{{\bm{\Delta}}}
\def\brho{{\bm{\rho}}}
\def\bphi{{\bm{\phi}}}
\def\bPhi{{\bm{\Phi}}}
\def\bpsi{{\bm{\psi}}}
\def\bzeta{{\bm{\zeta}}}
\def\bGamma{{\bm{\Gamma}}}
\def\blambda{{\bm{\lambda}}}

\def\btheta{{\bm{\theta}}}

\def\bv{{\bm{v}}}

\def\btheta{{\bm{\theta}}}
\def\bV{{\bm{V}}}

\def\bZ{{\bm{Z}}}
\def\bQ{{\bm{Q}}}
\def\bD{{\bm{D}}}
\def\bz{{\bm{z}}}
\def\bh{{\bm{h}}}

\def\bZ{{\bm{Z}}}
\def\bM{{\bm{M}}}
\def\beps{{\bm{\epsilon}}}
\def\bXi{{\bm{\Xi}}}

\def\diag{{\operatorname{diag}}}

\newcommand{\down}[1]{\lfloor #1 \rfloor}

\def\RB{{\mathbb R}}
\def\EB{{\mathbb E}}

\def\PB{{\mathbb P}}
\def\QB{{\mathbb Q}}

\def\NB{{\mathbb N}}

\makeatletter
\long\def\@makecaption#1#2{
  \vskip 0.8ex
  \setbox\@tempboxa\hbox{\small {\bf #1:} #2}
  \parindent 1.5em  
  \dimen0=\hsize
  \advance\dimen0 by -3em
  \ifdim \wd\@tempboxa >\dimen0
  \hbox to \hsize{
    \parindent 0em
    \hfil 
    \parbox{\dimen0}{\def\baselinestretch{0.96}\small
      {\bf #1.} #2
    } 
    \hfil}
  \else \hbox to \hsize{\hfil \box\@tempboxa \hfil}
  \fi
}
\makeatother

\newcommand{\cov}{\mathsf{Cov}}

%% file: intro.tex
Distributional reinforcement learning (DRL)~\citep{morimura2010nonparametric,bellemare2017distributional} extends the classical reinforcement learning framework by modeling the entire distribution of the return rather than only its expectation. This distributional perspective provides a richer characterization of the uncertainty associated with learning agents and has led to a series of successful applications in practice~\citep{bellemare2020autonomous,bodnar2019quantile}. 
Because return distributions are generally infinite-dimensional objects, practical DRL methods rely on finite-dimensional approximations. Among these approaches, quantile representations have become one of the most widely used paradigms, since~\citet{dabney2018distributional} proposed quantile temporal difference learning (QTD). By approximating return distributions through a collection of quantiles, quantile-based methods provide a flexible representation while retaining important distributional information, leading to algorithms such as quantile regression deep Q-network (QR-DQN) \citep{dabney2018distributional} and implicit quantile network (IQN) \citep{dabney2018implicit}.

Despite the empirical success of quantile-based distributional reinforcement learning, its statistical foundations remain relatively underdeveloped. In practical policy evaluation, the transition dynamics and reward distributions are unknown, and quantile representations must be estimated from finite observations. While these representations capture the variability and tail behavior of returns, their estimates are themselves subject to sampling uncertainty. Quantifying this uncertainty is important for assessing the reliability of estimated tail risks. 
Although substantial progress has been made in understanding the estimation and convergence of quantile-based methods, statistical inference from sequentially collected data remains less explored. Moreover, confidence intervals valid at a fixed time may not provide valid coverage when the reporting time is chosen based on the observed data. These considerations motivate two questions in online policy evaluation. At a fixed time, how can we construct statistically valid confidence intervals from the trajectory of an online learning algorithm? When observations are monitored sequentially, how can we construct confidence intervals that remain valid simultaneously over time?

In this paper, we develop two complementary approaches to online statistical inference for the quantile Bellman fixed point. First, we establish functional central limit theorems for QTD under synchronous, asynchronous and Markovian sampling settings. Using these limit theorems, we construct asymptotically valid confidence intervals by random scaling of the QTD trajectory. Second, we construct anytime-valid confidence intervals by multiscale Bellman-score inversion and propose an online approximation based on dyadic reward histograms. 

Addressing these questions presents distinct challenges. Unlike categorical distributional reinforcement learning, where the projected Bellman update is linear in the probability parameters, quantile-based methods involve nonlinear and nonsmooth operations arising from quantile projection. In particular, QTD is a nonlinear stochastic approximation algorithm with updates generated by indicator-type transformations of the quantile parameters, making the analysis of its sampling fluctuations substantially more involved. For anytime-valid inference, the quantile Bellman equation couples quantiles across states and levels, so inference for a single quantile generally requires accounting for the other coordinates. Our approach therefore constructs confidence intervals by inverting sequential tests over the joint quantile parameter space. The challenge is to support this inversion online without storing the entire observation history.

\subsection{Our Contributions}
We systematically study online statistical inference for quantile-based distributional policy evaluation in a tabular $\gamma$-discounted Markov decision process (MDP), where the state space $\gS$ and the action space $\gA$ are finite. Given a policy $\pi$ and quantile resolution $m$, let $\bm\eta_m$ denote the unique fixed point of the quantile-projected distributional Bellman operator $\bPi_m\gT^\pi$, where $\bPi_m$ is the quantile projection operator and $\gT^\pi$ is the distributional Bellman operator. Denoting the corresponding quantile parameter by $\btheta_m\in\RB^{\gS\times [m]}$, we aim to construct confidence intervals for linear functionals of $\btheta_m$ from sequentially collected data when the transition dynamics and reward distributions are unknown. We consider three sampling schemes: synchronous and asynchronous schemes from a generative model, and Markovian observations generated along a single trajectory under the policy $\pi$. 
We develop two online inference procedures providing asymptotically exact coverage at a fixed time and finite-sample simultaneous coverage over time, respectively, and establish a matching classwise lower bound for the interval-length rate of anytime-valid inference.

We first develop an asymptotic inference theory based on the QTD iterates. We establish functional central limit theorems for QTD under synchronous, asynchronous, and Markovian sampling in Theorems~\ref{thm:syn_QTD_FCLT}--\ref{thm:markov_QTD_FCLT}. The normalized partial-sum processes converge weakly to rescaled Brownian motions, with explicit asymptotic covariance matrices. Based on these results, Theorem~\ref{thm:QTD_random_scaling} develops an asymptotically pivotal random-scaling statistic that can be updated recursively without storing all past QTD iterates or estimating the asymptotic covariance matrix. As shown in Corollary~\ref{cor:QTD_random_scaling_CI}, the resulting confidence intervals achieve asymptotically exact coverage at time $T$, with lengths of order $T^{-1/2}$ in probability.


We next develop finite-sample anytime-valid inference directly from the quantile Bellman equation. We construct joint confidence sequences by inverting multiscale tests of conditional Bellman scores. Theorem~\ref{thm:av_width} establishes simultaneous coverage and shows that the resulting confidence intervals shrink at the iterated-logarithm rate $O(\sqrt{\log\log T/T})$ under positive limiting state-visit frequencies. To enable online inference without storing the entire observation history, we introduce dyadic reward histograms that require only $O(|\gS|^2T^\zeta)$ counters for any fixed $\zeta\in[1/2,1)$. Theorem~\ref{thm:dh_equivalence} establishes that the histogram-based intervals preserve simultaneous coverage and achieve the same leading interval-length rate, with an additional term of order $T^{-\zeta}$. Thus, the histogram approximation reduces storage requirements without affecting the asymptotic rate of interval length.

Finally, we establish a lower bound for anytime-valid inference. For any fixed quantile resolution and nonzero linear contrast, Theorem~\ref{thm:lil_lower} shows that every confidence interval sequence with simultaneous coverage must have realized length at least of order $\sqrt{\log\log T/T}$ infinitely often on its coverage event, for almost every parameter in a regular MDP submodel. Together with Theorem~\ref{thm:dh_equivalence}, this establishes the classwise optimality of the interval-length rate achieved by our dyadic histogram procedure.

\subsection{Related Work}

\paragraph{Distributional Reinforcement Learning.}
Distributional reinforcement learning studies the entire distribution of the return beyond its expectation \citep{bellemare2017distributional}. Several finite-dimensional representations have been developed, including categorical \citep{bellemare2017distributional}, quantile \citep{dabney2018distributional}, and generative-model-based representations \citep{freirich2019distributional,doan2018gan}. 
Among them, quantile representations have become a widely used paradigm. Starting from quantile temporal-difference learning \citep{dabney2018distributional}, subsequent works developed more flexible parametrizations, including implicit quantile networks \citep{dabney2018implicit} and fully parameterized quantile functions \citep{yang2019fully}.

The theoretical literature on distributional reinforcement learning has mainly focused on convergence and finite-sample guarantees. 
For categorical representations, \citet{rowland2018analysis} established asymptotic convergence of categorical temporal-difference learning, while \citet{peng2024statistical} derived non-asymptotic convergence rates and sample complexity guarantees. Furthermore, \citet{peng2025finite} characterized its finite-sample behavior under linear function approximation. 
For quantile-based methods, \citet{rowland2023analysis} established asymptotic convergence of quantile temporal-difference learning. However, their analysis focuses on consistency and does not characterize the asymptotic distribution of the iterates or provide statistical inference procedures.
Separately, \citet{rowland2023statistical} decomposed the estimation error of quantile temporal-difference learning into fixed-point bias and finite-sample variance, showing that it can outperform classical temporal difference learning for expected return estimation.

\paragraph{Statistical Analysis of Reinforcement Learning.}
Statistical inference for reinforcement learning has received increasing attention, with most existing work focusing on value functions. 
Representative results include high-confidence bounds for off-policy evaluation \citep{thomas2015high,jiang2016doubly}, bootstrap-based inference \citep{hao2021bootstrapping}, asymptotic inference for robust value functions \citep{yang2022toward}, confidence intervals for value functions under policy evaluation and policy learning \citep{shi2022statistical,zhu2023uncertainty}, and online inference procedures \citep{li2023statistical,li2023online}. More recently, \citet{bonnerjee2026recipe} developed anytime-valid error bounds for iterative algorithms, including temporal-difference learning.

Comparatively fewer works study inference for return distributions and other distributional quantities. 
\citet{chandak2021universal} and \citet{huang2022off} developed procedures for estimating return distribution functions and constructing confidence bands, while \citet{Qi03072025} studied distributional off-policy evaluation under offline data. 
Under a generative model, \citet{zhang2025estimation} developed estimation and inference theory for return distributions, and \citet{cheng2026statistical} studied model-based estimation and inference for quantile representations. 
In contrast, we develop online inference for the quantile Bellman fixed point through two complementary approaches, using random scaling of QTD iterates and anytime-valid confidence intervals based on multiscale Bellman-score inversion.

\paragraph{Anytime-Valid Inference.}
\citet{ramdas2023savi} developed a game-theoretic framework for anytime-valid inference based on test martingales and e-processes, encompassing sequential hypothesis testing and confidence sequences valid under optional stopping. \citet{howard2021confidence} established non-asymptotic, time-uniform concentration inequalities using mixture and stitching techniques, yielding confidence sequences with widths of order $O(\sqrt{\log\log T/T})$. \citet{waudby2024estimating} developed betting-based confidence intervals and confidence sequences for bounded means, including empirical-Bernstein bounds with asymptotically optimal widths. For lower bounds, \citet{shekhar2023betting} established instance-dependent lower bounds on the effective width of confidence sequences for bounded means. \citet{duchi2024timeuniform} established information-theoretic limits for time-uniform estimation through reductions to sequential hypothesis testing, showing that the $\sqrt{\log\log T/T}$ rate is unavoidable in models such as location families and logistic regression. In this paper, we establish a classwise lower bound at the iterated-logarithm rate for the realized lengths of arbitrary anytime-valid confidence intervals for quantile Bellman fixed points.


The remainder of this paper is organized as follows. 
In Section~\ref{Section:preliminary}, we introduce some basic concepts of distributional reinforcement learning.
In Section~\ref{Section:analysis}, we present the trajectory-based and anytime-valid inference methods and their guarantees.
In Section~\ref{Section:proof_outline}, we give proofs of the main results in Section~\ref{Section:analysis}.
Section~\ref{Section:numerical_experiments} presents numerical experiments, and Section~\ref{Section:conclusion} concludes our work.
Details of the proof are given in the appendices.

%% file: preliminary.tex
In this section, we introduce the necessary background for our work. We review the Markov decision process in Section~\ref{subsec:problem_setting}, followed by metrics on the space of measures in Section~\ref{subsec:measure_metric}. We introduce the distributional Bellman operator and quantile projection operator in Section~\ref{subsec:dist_bellman_quantile_projection}. 
We describe quantile temporal difference learning algorithms in Section~\ref{sec:qtdl}.
Finally, in Section~\ref{subsec:assumptions}, we present the assumptions under which our main theoretical results are established.

\subsection{Problem Setup}\label{subsec:problem_setting}

We consider a discounted Markov decision process (MDP) specified by 
the tuple $\gM=\<\gS,\gA,\gP_R,P,\gamma\>$, where $\gS$ and $\gA$ are finite state and action spaces, respectively, $\gP_R\colon\gS\times\gA\to\Delta([0,1])$ is the distribution of rewards, $P\colon\gS\times\gA\to\Delta(\gS)$ is the transition probability, and $\gamma\in(0,1)$ is the discount factor.
Here $\Delta(\cdot)$ denotes the set of probability distributions over some set.

For a fixed policy $\pi\colon\gS\to\Delta(\gA)$ and an initial state $S_0= s\in\gS$, a random trajectory $\{(S_t,A_t,R_t)\}_{t=0}^\infty$ can be sampled from the MDP using the following procedure: 
\begin{equation*}
    \begin{aligned}
        A_t\mid S_t&\sim\pi(\cdot\mid S_t),\\
        R_t\mid (S_t,A_t)&\sim \gP_R({\cdot}\mid S_t,A_t),\\
        {S_{t+1}}\mid{(S_t,A_t)}&\sim P({\cdot}\mid{S_t,A_t}).\\
    \end{aligned}
\end{equation*}

The return of such a trajectory starting from state $s$ is 
defined as the random variable
\begin{equation*}
    G^\pi(s)\coloneqq \sum_{t=0}^\infty \gamma^t R_t,
\end{equation*}
which lies in $[0, (1-\gamma)^{-1}]$ almost surely. The value function $V^\pi(s)$ is defined by the expected return $\EB[G^\pi(s)]$. 
We further denote by $\eta^\pi(s) \in \Delta([0, (1-\gamma)^{-1}])$ 
the distribution of $G^\pi(s)$, and denote $\bm\eta^\pi = 
(\eta^\pi(s))_{s \in \gS}$ for the collection of 
return distributions across all states.


\subsection{Metrics on the Space of Measures}\label{subsec:measure_metric}
Denote the space of all probability distributions on $\RB$ as $\sP$. For $\mu\in\sP$, the cumulative distribution function is defined as $F_\mu(x)=\mu(-\infty,x]$, and the quantile function is defined as $F^{-1}_{\mu}(\tau)=\inf\{x\colon F_\mu(x)\geq \tau\}$. For $1\leq p<\infty$ and $\mu,\nu\in\sP$, the $p$-Wasserstein metric between $\mu$ and $\nu$ is defined as 
\begin{equation*}
    W_p(\mu, \nu)=\left( \int_0^1 \left|F^{-1}_\mu(t)-F^{-1}_\nu(t)\right|^p\mathrm{d}t\right)^{1 / p}, 
\end{equation*}
and the $\infty$-Wasserstein metric is defined as
\begin{equation*}
    W_\infty(\mu, \nu)=\sup_{t\in(0,1)} \left|F^{-1}_\mu(t)-F^{-1}_\nu(t)\right|. 
\end{equation*}
Moreover, for any extended metric $d\colon\sP\times\sP\to[0,\infty]$, we can define its supremum extension $\Bar{d}\colon \sP^\gS\times\sP^\gS\to[0,\infty]$ as
\[
\Bar{d}(\bm\eta,\bm\eta^\prime)=\sup_{s\in\gS}d(\eta(s),\eta^\prime(s)), 
\]
which is an extended metric on $\sP^\gS$. 


\subsection{Distributional Bellman Operator and Quantile Projection Operator}\label{subsec:dist_bellman_quantile_projection}

A fundamental property of the value function is that it satisfies 
the Bellman equation. Letting $\bm V^\pi \coloneqq (V^\pi(s))_{s\in\gS}$, we have for any $s\in\gS$, 
\begin{equation*}
    \begin{aligned}
            V^\pi(s)&=\brk{T^\pi\bm V^\pi}(s)\\
    &\coloneqq \EB_{A\sim\pi(\cdot\mid s), R\sim\gP_R(\cdot\mid s,A)}[R]+\gamma\EB_{A\sim\pi(\cdot\mid s),S^\prime\sim P(\cdot\mid s,A)} [V^\pi(S^\prime)]\\
    &=\sum_{a\in\gA}\pi(a\mid s)\int_0^1 r \gP_R(\mathrm{d}r\mid s,a)+\gamma\sum_{a\in\gA,s^\prime\in\gS} \pi(a\mid s)P(s^\prime\mid s,a)V^\pi(s^\prime).
    \end{aligned}
\end{equation*}
The operator $T^\pi\colon \RB^{\gS}\to \RB^{\gS}$ is referred to as the Bellman operator, and the Bellman equation characterizes $\bm V^\pi$ as its unique fixed point.

An analogous relationship holds for the return distributions $\bm\eta^\pi$, known as the distributional Bellman equation. That is, for each $s\in\gS$, 
\begin{equation*}
\begin{aligned}
        \eta^\pi(s)&=\brk{\gT^\pi\bm\eta^\pi}(s)\\
    &\coloneqq \EB_{A\sim\pi(\cdot\mid s), R\sim\gP_R(\cdot \mid s,A),S^\prime\sim P(\cdot\mid s,A)}\brk{\prn{b_{R,\gamma}}_\#\eta^\pi(S^\prime)}\\
    &=\sum_{a\in\gA,s^\prime\in\gS}\pi(a\mid s)P(s^\prime\mid s,a)\int_0^1 \prn{b_{r,\gamma}}_\#\eta^\pi(s^\prime)\gP_R(\mathrm{d}r\mid s,a).
\end{aligned}
\end{equation*}
Here $b_{r,\gamma}\colon \RB\to\RB$ denotes the affine map $b_{r,\gamma}(x)=r+\gamma x$, and $g_\#\mu$ is the pushforward of a measure $\mu$ under $g$, defined by $g_\#\mu(B) = \mu(g^{-1}(B))$ for all Borel sets $B$. 
The integral ${\int_0^1 \prn{b_{r,\gamma}}_\#\eta^\pi(s^\prime)\gP_R(\mathrm{d}r\mid s,a)}$ is defined in the sense that for any Borel set $B$,
\begin{equation*}
    \brk{\int_0^1 \prn{b_{r,\gamma}}_\#\eta^\pi(s^\prime)\gP_R(\mathrm{d}r\mid s,a)}(B)=\int_0^1 \brk{\prn{b_{r,\gamma}}_\#\eta^\pi(s^\prime)}(B)\gP_R(\mathrm{d}r\mid s,a). 
\end{equation*}
Recall that $\sP$ is the space of all probability measures on $\RB$, and $\gT^\pi\colon \sP^\gS\to \sP^\gS$ is referred to as the distributional Bellman operator, of which the fixed point is the return distribution $\bm\eta^\pi$.

We approximate $\bm\eta^\pi$ using a finite quantile-parameterized distribution. For $m\in\NB$, the space of all quantile-parameterized probability distributions is defined as
\begin{equation*}
    \sP_{m} \coloneqq  \left\{ \nu_{\bx}=\frac{1}{m}\sum_{i=1}^m \delta_{x_i} \mid  \bx=\prn{x_1, \ldots, x_m}^{\top}\in \RB^{m}, x_1\leq \ldots\leq x_m \right\}, 
\end{equation*}
and every $\nu\in\sP_m$ is a uniform mixture of Dirac measures. 
We define the quantile
projection operator $\Pi_{m}\colon\sP\to\sP_m$ as
\begin{equation*}
   \Pi_{m}\nu=\frac{1}{m}\sum_{i=1}^m\delta_{F^{-1}_{\nu}(\tau_i)}, 
\end{equation*}
where $\tau_i=\frac{2i-1}{2m}$. We lift $\Pi_m$ to the product space $\sP^\gS$ by defining
$(\bPi_m{\bm\eta})(s) \coloneqq  \Pi_m\eta(s)$ for any $\bm\eta=(\eta(s))_{s\in\gS}\in\sP^\gS$. 
The properties of $\gT^\pi$ and $\bPi_m$ are summarized in the following proposition.  
\begin{proposition}\emph{\cite[Proposition 4.15, Lemma 5.25]{bdr2022}}\label{prop:basic_properties}
The following statements hold:
\begin{itemize}
    \item $\gT^\pi$ is  $\gamma$-contractive under the $\bar{W}_p$ metric for every $p\in[1,\infty]$, namely for every $\bm\eta,\bm\eta^\prime\in\sP^\gS$, 
    \[
    \bar{W}_p(\gT^\pi\bm\eta,\gT^\pi\bm\eta^\prime)\leq\gamma\bar{W}_p(\bm\eta,\bm\eta^\prime); 
    \]
    \item $\bPi_m$ is non-expansive under the $\bar{W}_\infty$ metric, namely $\bar{W}_\infty(\bPi_m\bm\eta,\bPi_m\bm\eta^\prime)\leq \bar{W}_\infty(\bm\eta,\bm\eta^\prime)$ for every $\bm\eta,\bm\eta^\prime\in\sP^\gS$.  
\end{itemize}
\end{proposition}

It immediately follows from Proposition~\ref{prop:basic_properties} that the quantile projected Bellman operator $\bPi_{m}{\gT}^{\pi}$ 
is a $\gamma$-contraction in the Polish space $(\sP_m^\gS,\bar{W}_{\infty})$.
Hence, the quantile projected Bellman equation $\bm\eta=\bPi_{m}\gT^{\pi}\bm\eta$ admits a unique solution $\bm\eta_{m}$. 
Denote $[m]=\{1,\ldots,m\}$, and we define the quantile parameter $\bm\theta_m\in\RB^{\gS\times [m]}$ by 
\begin{equation*}
    \eta_m(s)=\frac{1}{m}\sum_{i=1}^m\delta_{\theta_m(s,i)},
\end{equation*}
with $\theta_m(s,1)\leq \cdots \leq\theta_m(s,m)$ for every $s\in\gS$. 

Finally, we need a generalized version of quantile projection operator. For $\blambda\in[0,1]^{\gS\times[m]}$, define $\bPi^\blambda_m$ by
\begin{equation*}
    (\bPi^\blambda_m\bm\eta)(s)=\frac1m\sum_{i=1}^m\delta_{(1-\lambda(s,i))F^{-1}_{\eta(s)}(\tau_i)+\lambda(s,i)\Bar{F}^{-1}_{\eta(s)}(\tau_i)},
\end{equation*}
where $\Bar{F}^{-1}_{\eta(s)}(\tau)=\inf\{x:F_{\eta(s)}(x)>\tau\}$. It is proved in~\citet{rowland2023analysis} that every $\bPi^\blambda_m$ is non-expansive under the $\Bar{W}_\infty$ metric. Therefore, every $\bPi^\blambda_m\gT^\pi$ has a unique fixed point.

\subsection{Quantile Temporal Difference Learning}
\label{sec:qtdl}

We initialize QTD at a fixed $\btheta^{(0)}\in\RB^{\gS\times[m]}$ and first describe synchronous sampling. In this setting, at every iteration $t$, for every $s\in\gS$, we sample 
\begin{equation*}
    a^{(t,s)}\sim\pi(\cdot\mid s), \prn{r^{(t,s)},s^{\prime(t,s)}}\sim Q_{s,a^{(t,s)}}
\end{equation*}
independently across states and rounds from a generative model and update
\begin{equation}\label{eq:synch_QTD}
    \theta^{(t)}(s,i)=\theta^{(t-1)}(s,i)+\alpha_{t-1}\brk{\tau_i-\frac{1}{m}\sum_{j=1}^m\ind\brc{r^{(t,s)}+\gamma\theta^{(t-1)}\prn{s^{\prime(t,s)},j}<\theta^{(t-1)}(s,i)}}, 
\end{equation}
where we define
\begin{equation*}
    Q_{s,a}=\gP_R(\cdot\mid s,a)\otimes P(\cdot\mid s,a). 
\end{equation*}
In the asynchronous setting, let $\mu\in\Delta(\gS)$ be a sampling distribution with $\mu(s)>0$ for all $s\in\gS$. At every iteration $t$, we independently sample 
\begin{equation*}
    s^{(t)}\sim \mu(\cdot),a^{(t)}\sim\pi(\cdot\mid s^{(t)}), \prn{r^{(t)},s^{\prime(t)}}\sim Q_{s^{(t)},a^{(t)}}
\end{equation*}
from a generative model and update
\begin{equation}\label{eq:asynch_QTD}
    \theta^{(t)}(s,i)=\theta^{(t{-}1)}(s,i)+\alpha_{t{-}1}\ind\brc{s=s^{(t)}}\brk{\tau_i {-}\frac{1}{m}\sum_{j=1}^m\ind\brc{r^{(t)} {+} \gamma\theta^{(t{-}1)}\prn{s^{\prime(t)},j}<\theta^{(t{-}1)}(s,i)}}. 
\end{equation}

Finally, we consider the Markovian data setting, where the observations are generated along a single trajectory under the fixed policy $\pi$. 
Let $s^{(1)}$ be drawn from an initial
distribution. At every iteration $t$, conditional on the current state $s^{(t)}$, we sample
\begin{equation*}
    a^{(t)}\sim\pi(\cdot\mid s^{(t)}),\prn{r^{(t)},s^{\prime(t)}}\sim Q_{s^{(t)},a^{(t)}},
\end{equation*}
and set
\begin{equation*}
    s^{(t+1)}=s^{\prime(t)}.
\end{equation*}
The QTD iterate is then updated according to
\begin{equation}\label{eq:markov_QTD}
    \theta^{(t)}(s,i)=\theta^{(t{-}1)}(s,i)+\alpha_{t{-}1}\ind\brc{s=s^{(t)}}\brk{\tau_i {-} \frac{1}{m}\sum_{j=1}^m\ind\brc{r^{(t)} {+} \gamma\theta^{(t-1)}\prn{s^{\prime(t)},j}<\theta^{(t-1)}(s,i)}}.
\end{equation}
Thus, the asynchronous and Markovian settings use the same update rule, but differ in the data-generating process: in Equation~\eqref{eq:asynch_QTD}, $\{s^{(t)}\}_{t\ge1}$ are sampled i.i.d. from $\mu$, whereas in Equation~\eqref{eq:markov_QTD}, $\{s^{(t)}\}_{t\ge1}$ evolves according to the Markov chain induced by the policy $\pi$ and the transition kernel. We use step sizes $\alpha_t=c(t+1)^{-a}$ with $a\in(1/2,1)$ in all three settings. 

\subsection{Main Assumptions}\label{subsec:assumptions}
For every $(s,a)\in\gS\times\gA$, we make the following assumption about $\gP_R(\cdot\mid s,a)$. 
\begin{assumption}\label{assump:density}
    For any $s\in\gS$, $a\in\gA$, $\gP_R(\cdot\mid s,a)$ is supported on $[0,1]$ and has a Lebesgue density $p_{s,a}$. Moreover, $p_{s,a}$ is Lipschitz continuous on $[0,1]$ and there exists a positive constant $C_0$ such that $0< p_{s,a}(x)\leq C_0$ for any $x\in(0,1)$. 
\end{assumption}

Write $F_{s,a}$ for the CDF of $\gP_R(\cdot\mid s,a)$ and extend $p_{s,a}$ by zero outside $[0,1]$. We also make an assumption on the locations of $\bm\theta_m$. 
\begin{assumption}\label{Assumption_no_boundary_quantile}
For every $s,s^\prime\in\gS$ satisfying
\begin{equation*}
    P^\pi(s^\prime\mid s)\coloneqq \sum_{a\in\gA}\pi(a\mid s)P(s^\prime\mid s,a)>0, 
\end{equation*}
we have
\[
\theta_m(s,i)-\gamma\theta_m(s^\prime,j)\notin\{0,1\}
\]
for every $i,j\in[m]$. 
\end{assumption}

This is a technical condition that excludes boundary-degenerate configurations in which Bellman-shifted quantile locations coincide with the endpoints of the reward support. In particular, it guarantees that the projected Bellman equation is continuously differentiable around the true parameter. We suppose that the two assumptions above hold throughout the paper.

For Markovian sampling, we need another assumption. 
\begin{assumption}\label{assump:recurrence}
    Consider the Markov chain with state space $\gS$ and transition dynamics $P^\pi$. We suppose that this Markov chain is irreducible and aperiodic. Therefore, it has a unique stationary distribution $d^\pi$ satisfying $\min_sd^\pi(s)>0$. 
\end{assumption}

\begin{remark}
\label{rem:joint-reward-transition}
The results also cover general joint reward--transition laws, provided that $\gP_R(\cdot\mid s,a,v)$ satisfies Assumption~\ref{assump:density} whenever $P(v\mid s,a)>0$. Introduce the auxiliary action space $\widetilde{\gA}\coloneqq\gA\times\gS$ and define
\[
\begin{aligned}
    \widetilde\pi((a,v)\mid s)&=\pi(a\mid s)P(v\mid s,a),\\
    \widetilde P(s'\mid s,(a,v))&=\ind\{s'=v\},\\
    \widetilde{\gP}_R(\rd r\mid s,(a,v))&=\gP_R(\rd r\mid s,a,v).
\end{aligned}
\]
For $P(v\mid s,a)=0$, choose any auxiliary reward law satisfying the same assumption. The successor state is now deterministic given the state and auxiliary action, so conditional independence holds, and it is easy to verify that Assumptions~\ref{Assumption_no_boundary_quantile} and~\ref{assump:recurrence} are unchanged. 
\end{remark}

%% file: analysis.tex
We develop two methods for online inference on $\btheta_m$. In Section~\ref{subsec:qtd_fclt}, we establish functional central limit theorems for QTD under synchronous, asynchronous, and Markovian sampling, which are then used in Section~\ref{subsec:random_scaling} to construct confidence intervals by random scaling. Sections~\ref{subsec:anytime_inference} and~\ref{subsec:dyadic_inference} develop confidence sequences by inverting sequential tests based on the quantile-projected  Bellman equation, providing finite-sample coverage guarantees that hold simultaneously over time and bounds on interval lengths. Section~\ref{subsec:lil_optimality} establishes a matching classwise lower bound for interval length rate of anytime inference.

\subsection{Functional Central Limit Theorems for QTD}\label{subsec:qtd_fclt}
Throughout this section and Section~\ref{subsec:random_scaling}, let $\bm B(\cdot)$ be a standard Brownian motion in $\RB^{\gS\times[m]}$ on $[0,1]$. Process convergence is in $D([0,1],\RB^{\gS\times[m]})$ with the Skorokhod $J_1$ topology. 
We first define the matrices appearing in the functional central limit theorems (FCLTs). 

For a distribution $\nu$ on $\gS$, $\bD_{\nu}$ is defined as $\diag_s\brc{\nu(s)\bm I_m}$. $\bG$ is defined as 
\begin{align*}
        &G_{(s,i),(s^\prime,j)}\\
        & =-\frac{\gamma }{m}\sum_{a\in\gA}\pi(a\mid s)P(s^\prime\mid s,a)p_{s,a}(\theta_m(s,i)-\gamma\theta_m(s^\prime,j))\\
        & \quad +\frac{\ind\{(s,i)=(s',j)\}}{m}\sum_{a\in\gA,\tilde{s}\in\gS}\sum_{k=1}^m \pi(a\mid s)P(\tilde{s}\mid s,a)p_{s,a}\left(\theta_m(s,i)-\gamma\theta_m(\tilde{s},k)\right).
\end{align*}
For $s\in\gS$ and $i\in[m]$, define
\begin{equation*}
    \xi_{s,i}(a,r,s^\prime)=\frac{1}{m}\sum_{j=1}^m\ind\brc{r+\gamma\theta_m(s^\prime,j)<\theta_m(s,i)}-\tau_i.
\end{equation*}
The covariance matrix $\bGamma$ is defined by
\begin{align*}
    \Gamma_{(s,i),(\tilde{s},k)}&=\ind\brc{s=\tilde{s}}\cov_{A\sim\pi(\cdot\mid s),(R,S^\prime)\sim Q_{s,A}}\brk{\xi_{s,i}(A,R,S^\prime),\xi_{\tilde{s},k}(A,R,S^\prime)}. 
\end{align*}

Finally, define the normalized partial-sum process
\begin{equation*}
    \bPhi_T(u)=\frac{1}{\sqrt{T}}\sum_{t=0}^{\lfloor Tu\rfloor}\prn{\bm\theta^{(t)}-\bm\theta_m} 
\end{equation*}
for $u\in[0,1]$, where $\bPhi_T(0)\coloneqq 0$. The following theorems show that this process converges in distribution to a scaled Brownian motion, with different covariance matrices under different sampling mechanisms.
\begin{theorem}\label{thm:syn_QTD_FCLT}
    Under the synchronous update rule~\eqref{eq:synch_QTD}, as $T\to\infty$, 
    \begin{equation*}
        \bPhi_T(\cdot)\Rightarrow \bG^{-1}\bm \Gamma^{\frac12}\bm B(\cdot). 
    \end{equation*}
\end{theorem}
\begin{theorem}\label{thm:asyn_QTD_FCLT}
    Under the asynchronous update rule~\eqref{eq:asynch_QTD}, as $T\to\infty$, 
    \begin{equation*}
        \bPhi_T(\cdot)\Rightarrow \bG^{-1}\bD_\mu^{-\frac12}\bGamma^{\frac12}\bm B(\cdot). 
    \end{equation*}
\end{theorem}
\begin{theorem}\label{thm:markov_QTD_FCLT}
    Suppose Assumption~\ref{assump:recurrence} holds and the observations are generated from the Markov chain. Under the Markov update rule~\eqref{eq:markov_QTD}, as $T\to\infty$, 
    \begin{equation*}
        \bPhi_T(\cdot)\Rightarrow \bG^{-1}\bD_{d^\pi}^{-\frac12}\bGamma^{\frac12}\bm B(\cdot).
    \end{equation*}
\end{theorem}
The main technical challenge arises from the nonsmooth QTD updates, whose mean drift admits a Taylor expansion with a quadratic remainder only in a neighborhood of $\bm\theta_m$. We therefore use a stopping-time argument to establish local moment bounds and combine them with almost sure convergence of the iterates to control the accumulated nonlinear remainder. For Markovian sampling, we additionally use a Poisson equation to decompose the dependent noise into a martingale difference and negligible remainder terms.

\subsection{Online Inference by Random Scaling}\label{subsec:random_scaling}
We use the FCLTs to construct confidence intervals for linear functionals of $\btheta_m$ at a fixed time $T$. Random scaling yields an asymptotically pivotal statistic, allowing us to construct confidence intervals without estimating the asymptotic variance. Let
\begin{equation*}
    \bar{\bm\theta}_t=\frac{1}{t}\sum_{k=1}^t \bm\theta^{(k)}
\end{equation*}
be the Polyak--Ruppert averaged QTD iterate. Define the random scaling matrix
\begin{equation*}
\begin{aligned}
    \widehat{\bV}_T&=\frac{1}{T^3}\sum_{t=1}^T\left(\sum_{k=1}^t \bm\theta^{(k)}-t\bar{\bm\theta}_T\right)\left(\sum_{k=1}^t \bm\theta^{(k)}-t\bar{\bm\theta}_T\right)^\top\\
    &=\frac{1}{T^3}\sum_{t=1}^Tt^2\prn{\bar{\bm\theta}_t-\bar{\bm\theta}_T}\prn{\bar{\bm\theta}_t-\bar{\bm\theta}_T}^{\top}
\end{aligned}
\end{equation*}

Let $B(\cdot)$ denote a one-dimensional standard Brownian motion on $[0,1]$
and define the random variable
\begin{equation*}
    \gV=\frac{B(1)}{\brk{\int_0^1\prn{B(u)-uB(1)}^2\rd u}^{1/2}}, 
\end{equation*}
which is pivotal. The following theorem establishes the limiting distribution under all three sampling mechanisms. Its proof is given in Section~\ref{sec:random_scaling_proof}.
\begin{theorem}\label{thm:QTD_random_scaling}
Consider one of the sampling mechanisms under the conditions of its corresponding FCLT in Theorems~\ref{thm:syn_QTD_FCLT}--\ref{thm:markov_QTD_FCLT}. For a fixed $\bm c\in\RB^{\gS\times[m]}$, define
\[
 \sigma_c^2=
 \begin{cases}
  \bm c^\top\bG^{-1}\bGamma\bG^{-\top}\bm c,
      &\text{synchronous},\\
  \bm c^\top\bG^{-1}\bD_\mu^{-1}\bGamma\bG^{-\top}\bm c,
      &\text{asynchronous},\\
  \bm c^\top\bG^{-1}\bD_{d^\pi}^{-1}\bGamma\bG^{-\top}\bm c,
      &\text{Markovian}.
 \end{cases}
\]
Suppose $\sigma_c^2>0$. Then, as $T\to\infty$,
\[
 \frac{\bm c^\top(\bar{\bm\theta}_T-\bm\theta_m)}
 {\sqrt{\bm c^\top\widehat{\bV}_T\bm c}}
 \Rightarrow\gV.
\]
The ratio is defined to be zero when its denominator vanishes.
\end{theorem}

An asymptotic confidence interval follows immediately.
\begin{corollary}\label{cor:QTD_random_scaling_CI}
Under the conditions of Theorem~\ref{thm:QTD_random_scaling}, fix $\alpha\in(0,1)$ and denote
\begin{equation*}
    \mathbf{CI}_T(\alpha)=\left[\bm c^\top\bar{\bm\theta}_T-q_{\alpha/2}\sqrt{\bm c^\top\widehat{\bV}_T\bm c},\bm c^\top\bar{\bm\theta}_T+q_{\alpha/2}\sqrt{\bm c^\top\widehat{\bV}_T\bm c}\right],
\end{equation*}
where $q_{\alpha/2}$ is the upper $\alpha/2$-quantile of $\gV$. Then,
\begin{align*}
    \lim_{T\to\infty}\PB\brk{\bm c^\top\bm\theta_m\in\mathbf{CI}_T(\alpha)}=1-\alpha.
\end{align*}
Moreover,
\[
 \sqrt T\,\operatorname{length}(\mathbf{CI}_T(\alpha))
 \Rightarrow 2q_{\alpha/2}\sigma_c
 \left[\int_0^1\{B(u)-uB(1)\}^2\rd u\right]^{1/2}.
\]
In particular, $\operatorname{length}(\mathbf{CI}_T(\alpha))=O_{\PB}(T^{-1/2})$.
\end{corollary}

The random scaling matrix has a recursive implementation. Define
\begin{equation*}
    \bA_t=\sum_{k=1}^t k^2\bar{\bm\theta}_k\bar{\bm\theta}_k^\top,\qquad \bm b_t=\sum_{k=1}^t k^2\bar{\bm\theta}_k,\qquad q_t=\sum_{k=1}^t k^2.
\end{equation*}
These quantities admit the recursive updates
\begin{align*}
    \bar{\bm\theta}_t&=\frac{t-1}{t}\bar{\bm\theta}_{t-1}+\frac{1}{t}\bm\theta^{(t)},&\bA_t&=\bA_{t-1}+t^2\bar{\bm\theta}_t\bar{\bm\theta}_t^\top,\\
    \bm b_t&=\bm b_{t-1}+t^2\bar{\bm\theta}_t,&q_t&=q_{t-1}+t^2.
\end{align*}
Consequently,
\begin{equation*}
    \what{\bV}_t=\frac{1}{t^3}\prn{\bA_t-\bar{\bm\theta}_t\bm b_t^\top-\bm b_t\bar{\bm\theta}_t^\top+q_t\bar{\bm\theta}_t\bar{\bm\theta}_t^\top}.
\end{equation*}
Thus, both the averaged QTD iterate and its random scaling matrix can be updated sequentially as new QTD iterates become available.


\subsection{Anytime-Valid Inference by Multiscale Betting}\label{subsec:anytime_inference}\label{subsec:anytime_width}
Our second approach constructs confidence sequences by testing the quantile Bellman equation at candidate parameters \citep{ramdas2023savi,howard2021confidence}. For each candidate, betting products accumulate evidence against the equation from observed samples. At the true parameter, these products are nonnegative martingales, allowing us to control the probability of incorrectly excluding it at any time.

Write $d=|\gS|m$, $V_{\max}=(1-\gamma)^{-1}$, and
\begin{equation*}
 \Theta_m=\brc{\btheta\in[0,V_{\max}]^{\gS\times[m]}:\theta(s,1)\le\cdots\le\theta(s,m)\text{ for every }s}.
\end{equation*}
For $\btheta\in\RB^{\gS\times[m]}$, let $\bm\eta_{\btheta}$ have components $\eta_{\btheta}(s)= m^{-1}\sum_{j=1}^m\delta_{\theta(s,j)}$, and define the score
\begin{equation}\label{eq:av_score}
 g_{s,i}(\btheta;a,r,s')=\frac1m\sum_{j=1}^m\ind\brc{r+\gamma\theta(s',j)<\theta(s,i)}-\tau_i.
\end{equation}
Its mean under $A\sim\pi(\cdot\mid s)$ and $(R,S')\sim Q_{s,A}$ is
\begin{equation*}
 h_{s,i}(\btheta)=F_{(\gT^\pi\bm\eta_{\btheta})(s)}(\theta(s,i))-\tau_i,\qquad \bh(\btheta_m)=\bm0.
\end{equation*}
For each round $t$ and state $s$, let $I_{t,s}=1$ under synchronous sampling and $I_{t,s}=\ind\{s^{(t)}=s\}$ otherwise. Whenever $I_{t,s}=1$, write $(A_{t,s},R_{t,s},S'_{t,s})$ for the observed transition, corresponding to $(a^{(t,s)},r^{(t,s)},s^{\prime(t,s)})$ under synchronous sampling and $(a^{(t)},r^{(t)},s^{\prime(t)})$ otherwise. Define
\begin{equation*}
Z_{t,s,i}(\btheta)=g_{s,i}(\btheta;A_{t,s},R_{t,s},S'_{t,s})\ind\{I_{t,s}=1\}.
\end{equation*}

For each coordinate $(s,i)$ and scale $k\ge0$, consider two betting products with a fixed betting fraction $\lambda_k\in(0,1/2]$:
\begin{equation}\label{eq:av_multiscale_wealth}
 K_{t,s,i,k}^{\pm}(\btheta)=\prod_{u=1}^t\bigl(1\pm\lambda_k Z_{u,s,i}(\btheta)\bigr).
\end{equation}
The two signs allow evidence to accumulate against deviations in either direction. To accommodate different sample sizes and establish bounds on confidence interval lengths, we combine betting products across multiple scales. Fix $\alpha\in(0,1)$ and choose
\begin{equation}\label{eq:av_multiscale_parameters}
 w_k=\frac{6}{\pi^2(k+1)^2},\qquad\lambda_k=\min\left\{\frac12,\sqrt{\frac1{2^k}\log\frac{4d}{\alpha w_k}}\right\}.
\end{equation}
At time $t\ge1$, put $J_t=\lfloor\log_2t\rfloor$ and define
\begin{align*}
 E_t(\btheta)&=\frac1{2d}\sum_{s\in\gS}\sum_{i=1}^m\sum_{k=0}^{J_t}w_k\prn{K_{t,s,i,k}^{+}(\btheta)+K_{t,s,i,k}^{-}(\btheta)},\\
 C_t&=\brc{\btheta\in\Theta_m:E_t(\btheta)<B_\alpha},
\end{align*}
where $B_\alpha=2/\alpha$, and take $E_0=1$ and $C_0=\Theta_m$. 

Finally, we convert the region to construct confidence intervals. For $C\subseteq\Theta_m$ and a contrast $\bm c\in\RB^d$, write
\begin{equation*}
 \mathcal I^c(C)=\brk{\inf_{\btheta\in C}\bm c^\top\btheta,\sup_{\btheta\in C}\bm c^\top\btheta},
 \qquad I_t^c=\mathcal I^c(C_t).
\end{equation*}
Here $\gI^c(C)=\varnothing$ if $C=\varnothing$ and empty intervals have length zero. The following theorem establishes anytime-valid coverage and an iterated-logarithm bound on the interval length.
\begin{theorem}\label{thm:av_width}
Fix $m\ge1$ and $\alpha\in(0,1)$. Suppose Assumption~\ref{assump:recurrence} holds under Markovian sampling. With probability at least $1-\alpha$, $\btheta_m\in C_t$ for every $t\ge0$, and for all sufficiently large $t$, simultaneously for all $\bm c\in\RB^d$,
\begin{equation}\label{eq:av_global_rate}
 \operatorname{length}(I_t^c)\le C\norm{\bm c}_1\sqrt{\frac{\log\log t}{t}},
\end{equation}
where the constant $C$ depends only on the MDP including the policy, $m$, and the sampling mechanism.
\end{theorem}
The anytime coverage follows from Ville's inequality applied to the betting martingales. For the interval-length bound, we notice that at a coordinate where a candidate parameter has its largest error, the error in the current quantile dominates the discounted errors in all successor quantiles. By CDF monotonicity, this allows us to compare its empirical Bellman score with that at a uniform shift of $\btheta_m$. Together with a time-uniform concentration bound, this comparison rules out candidates far from $\btheta_m$ and yields the stated rate.

\subsection{Online Inference by Dyadic Reward Histograms}\label{subsec:dyadic_inference}
The confidence region $C_t$ is defined implicitly in a high-dimensional parameter space, making its direct computation impractical. To obtain confidence intervals, one may instead search over candidate parameters, evaluating $E_t(\btheta)$ at new points as needed. Each such evaluation requires access to the full observation history, whose storage grows linearly with time. We replace this history by dyadic reward histograms and use the bin endpoints to bound the evidence conservatively.

Fix $\zeta\in[1/2,1)$. For each epoch $j\ge0$, set
\begin{equation}\label{eq:dh_epochs}
 \gT_j=\brc{2^j,\ldots,2^{j+1}-1},\qquad L_j=2^{\lceil\zeta j\rceil},\qquad \Delta_j= L_j^{-1}.
\end{equation}
Partition $[0,1]$ into bins $B_{j,\ell}=[\ell\Delta_j,(\ell+1)\Delta_j)$, $0\le\ell<L_j$, including $1$ in the last bin. Write $b_{j,\ell}^-$ and $b_{j,\ell}^+$ for their endpoints, and we group transitions by their reward bins: 
\begin{equation*}
 H_{s,v,j,\ell}(t)=\sum_{u\le t:\,u\in\mathcal T_j}I_{u,s}\ind\{S^\prime_{u,s}=v,\ R_{u,s}\in B_{j,\ell}\}.
\end{equation*}
For each candidate $\btheta$, the score of a transition in this bin lies between
\begin{align}
 g^-_{s,i;v,j,\ell}(\btheta)&=\frac1m\sum_{q=1}^m\ind\{b_{j,\ell}^++\gamma\theta(v,q)<\theta(s,i)\}-\tau_i,\label{eq:dh_score_lower}\\
 g^+_{s,i;v,j,\ell}(\btheta)&=\frac1m\sum_{q=1}^m\ind\{b_{j,\ell}^-+\gamma\theta(v,q)<\theta(s,i)\}-\tau_i.\label{eq:dh_score_upper}
\end{align}
Using the weights and bets in~\eqref{eq:av_multiscale_parameters}, define
\begin{align}
 \underline K^+_{t,s,i,k}(\btheta)&=\prod_{v,j,\ell}[1+\lambda_k g^-_{s,i;v,j,\ell}(\btheta)]^{H_{s,v,j,\ell}(t)},\label{eq:dh_products}\\
 \underline K^-_{t,s,i,k}(\btheta)&=\prod_{v,j,\ell}[1-\lambda_k g^+_{s,i;v,j,\ell}(\btheta)]^{H_{s,v,j,\ell}(t)},\nonumber\\
 \underline E_t(\btheta)&=\frac1{2d}\sum_{s,i}\sum_{k=0}^{J_t}w_k[\underline K^+_{t,s,i,k}(\btheta)+\underline K^-_{t,s,i,k}(\btheta)].\label{eq:dh_wealth}
\end{align}
The products range over $v\in\gS$, $0\le j\le J_t$, and $0\le\ell<L_j$. Set $\underline E_0=1$ and
\begin{equation}\label{eq:dh_region}
 \widetilde C_t=\{\btheta\in\Theta_m:\underline E_t(\btheta)<B_\alpha\},
 \qquad \widetilde I_t^c=\mathcal I^c(\widetilde C_t).
\end{equation}
Algorithm~\ref{alg:dyadic_histogram} retains $O(|\gS|^2t^\zeta)$ counters, shared by all coordinates and candidates. Evaluation and projection from these counts are described in Appendix~\ref{app:dh_implementation}.

\begin{algorithm}[htbp]
\caption{Online confidence intervals by dyadic reward histograms}
\label{alg:dyadic_histogram}
\begin{algorithmic}[1]
\REQUIRE $\alpha\in(0,1)$, $\zeta\in[1/2,1)$, resolution $m$, discount $\gamma$, and contrast $\bm c$.
\STATE Initialize $t=0$ and all counts to zero; report $\widetilde I_0^c=\mathcal I^c(\Theta_m)$.
\LOOP
 \STATE Observe round $t+1$, set $t\leftarrow t+1$, and let $j=\lfloor\log_2t\rfloor$.
 \STATE Initialize the epoch-$j$ arrays if needed, using $L_j$ and $\Delta_j$ in~\eqref{eq:dh_epochs}.
 \FOR{each $s$ with $I_{t,s}=1$}
  \STATE Read $(R_{t,s},S'_{t,s})=(r,v)$ and set $\ell=\min\{\lfloor r/\Delta_j\rfloor,L_j-1\}$.
  \STATE Increment $H_{s,v,j,\ell}$ by one and discard the transition.
 \ENDFOR
 \STATE Evaluate $\underline E_t(\btheta)$ from~\eqref{eq:dh_products}--\eqref{eq:dh_wealth} and report $\widetilde I_t^c$ from~\eqref{eq:dh_region}.
\ENDLOOP
\end{algorithmic}
\end{algorithm}

\begin{theorem}\label{thm:dh_equivalence}
Under the assumptions of Theorem~\ref{thm:av_width}, fix $\zeta\in[1/2,1)$. There exist $C_1,C_2>0$ such that, with probability at least $1-\alpha$, $\btheta_m\in\widetilde C_t$ for every $t\ge0$ and, for all sufficiently large $t$, simultaneously for all $\bm c\in\RB^d$,
\begin{equation}\label{eq:dh_length}
 \operatorname{length}(\widetilde I_t^c)
 \le\norm{\bm c}_1\left(
 C_1\sqrt{\frac{\log\log t}{t}}
 +C_2t^{-\zeta}\right).
\end{equation}
The constant $C_1$ depends only on the MDP, the policy, $m$, and the sampling mechanism, and $C_2$ may additionally depend on $\zeta$.
\end{theorem}

For every $\zeta\in[1/2,1)$, the additional error from the histogram approximation is of smaller order, so the resulting confidence intervals have the same $O(\sqrt{\log\log t/t})$ length as those constructed in Section~\ref{subsec:anytime_inference}. 
The histogram approximation preserves anytime coverage because the resulting betting evidence is bounded above by the original evidence. To control the interval lengths, we observe that each comparison threshold can fall inside at most one reward bin per epoch. Bounding the number of observations in these bins uniformly over candidate parameters, and applying the same maximal-coordinate comparison as in Section~\ref{subsec:anytime_inference}, shows that the resulting intervals have the same $O(\sqrt{\log\log t/t})$ length for every $\zeta\in[1/2,1)$. The proof is given in Section~\ref{sec:dh_proofs}.

\subsection{Classwise Optimality of the Iterated-Logarithm Rate}
\label{subsec:lil_optimality}

In this section, we establish a lower bound for confidence-interval sequences based on the full observation history. Write
\begin{equation*}
 \rho_t\coloneqq\sqrt{\frac{2\log\log t}{t}},\qquad t\ge3.
\end{equation*}
An interval rule is a single procedure that returns $I_t$ from observations through time $t$, possibly using parameter-independent randomization. Empty intervals are allowed and have length zero, and unbounded intervals have infinite length.

\begin{theorem}\label{thm:lil_lower}
Fix nonempty finite sets $\gS,\gA$, a policy $\pi$, $m\ge1$, $\gamma\in(0,1)$, and $\bm c\in\RB^{\gS\times[m]}\setminus\{\bm0\}$. There exist $b,C>0$ and a family of MDPs $\{\gM_\beta:\beta\in J\}$, with $J\coloneqq[-b,b]$, satisfying Assumptions~\ref{assump:density}--\ref{assump:recurrence}, such that $\psi(\beta)\coloneqq\bm c^\top\btheta_m(\gM_\beta)$ is continuously differentiable on a neighborhood of $J$ and $\inf_{\beta\in J}|\psi'(\beta)|\ge C/2$.
Consider any of the three sampling mechanisms, where the initial distribution does not depend on $\beta$ under Markovian sampling. For any $\alpha\in(0,1)$ and any interval rule $(I_t)_{t\ge1}$ satisfying
\begin{equation*}
 A_\beta\coloneqq\{\psi(\beta)\in I_t\text{ for every }t\ge1\},
 \qquad \inf_{\beta\in J}\PB_\beta(A_\beta)\ge1-\alpha,
\end{equation*}
the following hold for Lebesgue-almost every $\beta\in J$:
\begin{align}
 \PB_\beta\!\left(
 A_\beta\cap\left\{
 \limsup_{t\to\infty}\frac{\operatorname{length}(I_t)}{\rho_t}
 <2|\psi'(\beta)|\right\}\right)&=0,
 \label{eq:lil_on_coverage}\\
 \PB_\beta\!\left(
 \limsup_{t\to\infty}\frac{\operatorname{length}(I_t)}{\rho_t}\ge C
 \right)&\ge1-\alpha.
 \label{eq:lil_probability_lower}
\end{align}
\end{theorem}

Together with Theorem~\ref{thm:dh_equivalence}, Theorem~\ref{thm:lil_lower} establishes the classwise optimality of the iterated-logarithm rate achieved by the dyadic histogram intervals. The lower bound concerns the realized lengths of arbitrary anytime-valid confidence intervals, even when the full observation history is available. The proof transfers a pathwise lower bound for Gaussian mean estimation to a regular family of MDPs, as outlined in Section~\ref{sec:lil_proof}.


%% file: proof_outline.tex
In this section, we prove the QTD functional central limit theorems, the random-scaling guarantees, and the anytime inference results in Section~\ref{Section:analysis}. 

\subsection{Functional CLTs for Synchronous and Asynchronous QTD}
To establish the FCLTs for synchronous and asynchronous QTD, we first present the following general functional central limit theorem, which is proved in Appendix~\ref{app:independent_fclt}. 
\begin{theorem}\label{thm:general_FCLT}
    Suppose that $(\bZ_t)_{t\ge0}$ is defined recursively, for $t\ge1$, by
    \begin{equation*}
        \bZ_t=\bZ_{t-1}-\alpha_{t-1}[\bh(\bZ_{t-1})+\beps_{t-1}], 
    \end{equation*}
    where $\bh$ is a Borel function with a unique zero $\bz^*$.
    Suppose there exist a matrix $\bM$ and constants $L,\delta>0$ such that
    \begin{equation*}
        \norm{\bh(\bz)-\bM(\bz-\bz^*)}\leq L\norm{\bz-\bz^*}^2
    \end{equation*}
    for $\|\bz-\bz^*\|\leq\delta$, and $\bM$ is a uniformly repulsive matrix (all its eigenvalues have positive real parts). Let $(\gF_t)_{t\geq-1}$ be a filtration such that $\bZ_t$ is $\gF_{t-1}$-measurable, and suppose that 
    \begin{enumerate}[(i)]
        \item $\beps_t$ is $\gF_t$-measurable; 
        \item $\EB(\beps_t\mid\gF_{t-1})=0$;
        \item there exists $L_1>0$ such that $\|\beps_t\|\leq L_1$ almost surely; 
        \item $\EB(\beps_t\beps_t^\top\mid\gF_{t-1})\to\bGamma$ almost surely. 
    \end{enumerate}
    Set $\alpha_t=c(t+1)^{-a}$ with $1/2<a<1$ and suppose that $\bZ_n\to \bz^*$ almost surely. Then
    \begin{equation*}
        \bPhi_T(\cdot)\coloneqq\frac{1}{\sqrt{T}}\sum_{t=0}^{\lfloor T\cdot\rfloor}(\bZ_t-\bz^*)\Rightarrow \bM^{-1}\bGamma^{\frac12}\bB(\cdot),
    \end{equation*}
    where $\bB(\cdot)$ is a $d$-dimensional standard Brownian motion. 
\end{theorem}

To apply Theorem~\ref{thm:general_FCLT}, we verify the uniqueness of the fixed point for the quantile Bellman equation, the stability and local regularity of its linearization, and the almost sure convergence of QTD.

Define the mean Bellman score by
\begin{equation*}
 h_{s,i}(\btheta)
 =\frac1m\sum_{a,s',j}\pi(a\mid s)P(s'\mid s,a)
 F_{s,a}\bigl(\theta(s,i)-\gamma\theta(s',j)\bigr)-\tau_i,
\end{equation*}
and write $F_s^\star=F_{(\gT^\pi\bm\eta_m)(s)}$, $f_s^\star=p_{(\gT^\pi\bm\eta_m)(s)}$, where
\begin{equation*}
 f_s^\star(x)=\frac1m\sum_{a,s',j}\pi(a\mid s)P(s'\mid s,a)
 p_{s,a}\bigl(x-\gamma\theta_m(s',j)\bigr).
\end{equation*}
The argument of \citet[Lemma A.2]{cheng2026statistical} applies at the fixed resolution $m$ and gives $f_s^\star(\theta_m(s,i))>0$. By Assumptions~\ref{assump:density} and~\ref{Assumption_no_boundary_quantile}, $f_s^\star$ is continuous near these quantiles. Since there are finitely many coordinates, there exist $\kappa_m,r_m>0$, depending only on the MDP, the policy, and $m$, such that
\[
 f_s^\star(x)\ge\kappa_m
 \qquad\text{whenever }|x-\theta_m(s,i)|\le r_m.
\]
Integrating gives, for $0\le a\le r_m$,
\begin{equation}\label{eq:av_cdf_margin}
 F_s^\star(\theta_m(s,i)+a)-\tau_i\ge\kappa_m a,\qquad
 F_s^\star(\theta_m(s,i)-a)-\tau_i\le-\kappa_m a.
\end{equation}
Thus, the left and right quantiles at every $\tau_i$ coincide, and
\begin{equation}\label{eq:common_quantile_fixed_point}
 \bPi_m^{\blambda}\gT^\pi\bm\eta_m=\bm\eta_m,
 \qquad \blambda\in[0,1]^{\gS\times[m]}.
\end{equation}
By \citet{rowland2023analysis}, each projected operator has a unique fixed point. Since every zero of $\bh$ corresponds to a fixed point of one of these operators, $\btheta_m$ is the unique zero of $\bh$.

We next examine the Jacobian of $\bh$. A direct differentiation gives $\nabla\bh(\btheta_m)=\bG$, where
\begin{equation}\label{eq:bellman_jacobian_entries}
 G_{(s,i),(s',j)}
 =\ind\{(s,i)=(s',j)\}f_s^\star(\theta_m(s,i))
 -\frac{\gamma}{m}\sum_a\pi(a\mid s)P(s'\mid s,a)
 p_{s,a}\bigl(\theta_m(s,i)-\gamma\theta_m(s',j)\bigr).
\end{equation}
The off-diagonal entries are nonpositive, and summing each row gives
\begin{equation}\label{eq:bellman_row_sum}
 G_{(s,i),(s,i)}
 -\sum_{(s',j)\ne(s,i)}|G_{(s,i),(s',j)}|
 =(1-\gamma)f_s^\star(\theta_m(s,i))>0.
\end{equation}
Thus, $\bG$ is strictly diagonally dominant, and Gershgorin's theorem yields
\[
 \operatorname{Re}\lambda
 \ge(1-\gamma)\min_{s,i}f_s^\star(\theta_m(s,i))>0
\]
for every eigenvalue $\lambda$ of $\bG$. Positive row scaling preserves strict diagonal dominance. Hence $\bD_\mu\bG$ is also uniformly repulsive when $\mu(s)>0$ for every $s$, as is $\bD_{d^\pi}\bG$ under Assumption~\ref{assump:recurrence}.

Furthermore, Assumptions~\ref{assump:density} and~\ref{Assumption_no_boundary_quantile} imply that $\nabla\bh$ is locally Lipschitz near $\btheta_m$. Taylor's formula therefore gives
\begin{equation}\label{eq:bellman_quadratic_remainder}
 \norm{\bh(\btheta)-\bG(\btheta-\btheta_m)}
 \le L\norm{\btheta-\btheta_m}^2
\end{equation}
for $\btheta$ in a neighborhood of $\btheta_m$.

Finally, the almost sure convergence required by Theorem~\ref{thm:general_FCLT} is established in Appendix~\ref{app:qtd_consistency}.
\begin{lemma}\label{lem:consistency}
For update rules~\eqref{eq:synch_QTD} and~\eqref{eq:asynch_QTD}, as $T\to\infty$, $\btheta^{(T)}\to\btheta_m$ almost surely.
\end{lemma}

For synchronous QTD, define
\[
 H^{(t)}_{s,i}(\btheta)
 =\frac1m\sum_{j=1}^m
 \ind\brc{r^{(t,s)}+\gamma\theta(s^{\prime(t,s)},j)<\theta(s,i)}-\tau_i.
\]
Since $\EB[\bH^{(t)}(\btheta)]=\bh(\btheta)$, the recursion takes the form
\[
 \btheta^{(t)}=\btheta^{(t-1)}-\alpha_{t-1}
 \{\bh(\btheta^{(t-1)})+\beps_{t-1}\},
 \qquad
 \beps_{t-1}\coloneqq\bH^{(t)}(\btheta^{(t-1)})-\bh(\btheta^{(t-1)}).
\]
Let $\gF_t$ be generated by the initialization and all observations through iteration $t+1$. Then $\beps_t$ is $\gF_t$-measurable and $\EB[\beps_t\mid\gF_{t-1}]=0$, since each new batch is independent of the past. The noise is uniformly bounded, and Lemma~\ref{lem:consistency}, together with the continuity of its conditional covariance, gives
\[
 \EB[\beps_t\beps_t^\top\mid\gF_{t-1}]
 \to\bGamma\qquad\text{almost surely}.
\]
The uniqueness and local regularity of $\bh$, together with the positive stability of $\bG$, verify the remaining conditions of Theorem~\ref{thm:general_FCLT}. Applying that theorem proves Theorem~\ref{thm:syn_QTD_FCLT}.

For asynchronous QTD, define
\[
 H^{(t)}_{s,i}(\btheta)
 =\ind\{s^{(t)}=s\}
 \left[
 \frac1m\sum_{j=1}^m
 \ind\brc{r^{(t)}+\gamma\theta(s^{\prime(t)},j)<\theta(s,i)}
 -\tau_i
 \right].
\]
Now $\EB[\bH^{(t)}(\btheta)]=\bD_\mu\bh(\btheta)$, so
\[
 \btheta^{(t)}
 =\btheta^{(t-1)}-\alpha_{t-1}
 \{\bD_\mu\bh(\btheta^{(t-1)})+\beps_{t-1}\},
 \qquad
 \beps_{t-1}\coloneqq
 \bH^{(t)}(\btheta^{(t-1)})
 -\bD_\mu\bh(\btheta^{(t-1)}).
\]
Since $\mu$ has full support, $\bD_\mu\bh$ has the same unique zero $\btheta_m$, and its Jacobian $\bD_\mu\bG$ is uniformly repulsive. The local quadratic remainder bound follows from~\eqref{eq:bellman_quadratic_remainder}. The noise conditions are verified as in the synchronous case, with the corresponding filtration and
\[
 \EB[\beps_t\beps_t^\top\mid\gF_{t-1}]
 \to\bD_\mu\bGamma\qquad\text{almost surely}.
\]
Thus, Theorem~\ref{thm:general_FCLT} applies with drift matrix $\bD_\mu\bG$ and noise covariance $\bD_\mu\bGamma$, proving Theorem~\ref{thm:asyn_QTD_FCLT}.

\subsection{Functional CLT for Markovian QTD}
To prove Theorem~\ref{thm:markov_QTD_FCLT}, we need to establish the consistency result first. 
\begin{lemma}\label{lem:consistency_markov}
    For the update rule~\eqref{eq:markov_QTD}, as $T\to\infty$, $\btheta^{(T)}\to\btheta_m$ almost surely. 
\end{lemma}
The proof in Appendix~\ref{app:markov_consistency} uses the stochastic-approximation framework of~\citet{benaim2006dynamics} and the Lyapunov function of~\citet{rowland2023analysis}. To decompose the Markovian noise, write $\bD_s\coloneqq\bD_{\delta_s}$ and use the following Poisson equation. Its proof is in Appendix~\ref{app:sa_technical}. 
\begin{proposition}\label{prop:poisson_eq}
    For every $\btheta$ there exists a unique $\bv(\btheta,\cdot)\in(\RB^{\gS\times[m]})^\gS$ such that
    \begin{align*}
        \bv(\btheta,s)-\gP^\pi\bv(\btheta,s)&=(\bD_s-\bD_{d^\pi})\bh(\btheta)\coloneqq\bq(\btheta,s),\\
        \sum_{s\in\gS}d^\pi(s)\bv(\btheta,s)&=0,
    \end{align*}
    where
    \begin{equation*}
        \gP^\pi\bv(\btheta,s)=\sum_{s^\prime\in\gS}P^\pi(s^\prime\mid s)\bv(\btheta,s^\prime). 
    \end{equation*}
    Moreover, there exists $C<\infty$ such that
\[
    \sup_{\btheta}\max_{s\in\gS}\norm{\bv(\btheta,s)}\le C.
\]
\end{proposition}
Denote $\bDelta^{(t)}=\btheta^{(t)}-\btheta_m$ and define
\begin{equation*}
    H^{(t)}_{s,i}(\btheta)=\ind\brc{s=s^{(t)}}\brk{\frac{1}{m}\sum_{j=1}^m\ind\brc{r^{(t)}+\gamma\theta\prn{s^{\prime(t)},j}<\theta(s,i)}-\tau_i}.
\end{equation*}
Then we have the following decomposition:
\begin{align}
    \bDelta^{(t)}
    =&(\bI-\alpha_{t-1}\bM)\bDelta^{(t-1)}\notag\\
    &-\alpha_{t-1}
    \begin{aligned}[t]
        \bigg[&\bD_{d^\pi}\bh(\btheta^{(t-1)})-\bM\bDelta^{(t-1)}\\
        &+\bv(\btheta^{(t-1)},s^{(t)})-\bv(\btheta^{(t-1)},s^{(t+1)})\\
        &+\bv(\btheta^{(t-1)},s^{(t+1)})-\gP^\pi\bv(\btheta^{(t-1)},s^{(t)})\\
        &+\bH^{(t)}(\btheta^{(t-1)})-\bH^{(t)}(\btheta_m)-\bD_{s^{(t)}}\bh(\btheta^{(t-1)})\\
        &+\bH^{(t)}(\btheta_m)\bigg]
    \end{aligned}\notag\\
    \coloneqq&(\bI-\alpha_{t-1}\bM)\bDelta^{(t-1)}-\alpha_{t-1}\prn{\brho_{t-1}+\bw_{t-1}+\bphi_{t-1}+\bpsi_{t-1}+\bzeta_{t-1}}\label{eq:markov_decomposition},
\end{align}
Here \(\bM=\bD_{d^\pi}\bG\) is uniformly repulsive by~\eqref{eq:bellman_row_sum}. For \(t\ge-1\), define
\[
 \gF_t=\sigma\bigl(\btheta^{(0)},s^{(1)},
 \{a^{(u)},r^{(u)},s^{(u+1)}:1\le u\le t+1\}\bigr).
\]
Then \(\bphi_t,\bpsi_t,\bzeta_t\) are martingale differences with respect to \((\gF_t)\). Set \(\beps_t=\bphi_t+\bpsi_t+\bzeta_t\) and introduce
\(\widetilde{\bDelta}^{(t)}\coloneqq\bDelta^{(t)}-\alpha_t\bv(\btheta^{(t)},s^{(t+1)})\).
This auxiliary sequence satisfies
\begin{equation*}
    \wtilde{\bDelta}^{(t)}=(\bI-\alpha_{t-1}\bM)\wtilde{\bDelta}^{(t-1)}-\alpha_{t-1}(\wtilde{\brho}_{t-1}+\beps_{t-1}), 
\end{equation*}
where
\begin{align*}
    \wtilde{\brho}_t={}&\brho_t+\alpha_t\bM\bv(\btheta^{(t)},s^{(t+1)})
    +\bv(\btheta^{(t+1)},s^{(t+2)})-\bv(\btheta^{(t)},s^{(t+2)})\\
    &-\prn{1-\frac{\alpha_{t+1}}{\alpha_t}}\bv(\btheta^{(t+1)},s^{(t+2)}).
\end{align*}
The following lemma, proved in Appendix~\ref{app:aux_fclt}, gives the limit of the corrected recursion. 
\begin{lemma}\label{lem:aux_FCLT}
    As $T\to\infty$, 
    \begin{equation*}
        \frac{1}{\sqrt{T}}\sum_{t=0}^{\lfloor T\cdot\rfloor}\wtilde{\bDelta}^{(t)}\Rightarrow\bG^{-1}\bD_{d^\pi}^{-\frac12}\bGamma^{\frac12}\bm B(\cdot).
    \end{equation*}
\end{lemma}
By the boundedness of \(\bv\) in Proposition~\ref{prop:poisson_eq},
\[
 \sup_{u\in[0,1]}\left\|T^{-1/2}\sum_{t=0}^{\lfloor Tu\rfloor}
 (\bDelta^{(t)}-\widetilde{\bDelta}^{(t)})\right\|
 \le C T^{-1/2}\sum_{t=0}^T\alpha_t\longrightarrow0
 \quad\text{almost surely}.
\]
Combining this estimate with Lemma~\ref{lem:aux_FCLT} and Slutsky's theorem proves Theorem~\ref{thm:markov_QTD_FCLT}.

\subsection{Random-Scaling Inference}\label{sec:random_scaling_proof}\label{Appendix_omitted_proofs_3}
We give the argument for synchronous QTD and the corresponding FCLTs yield the other two cases. Define
\[
\widetilde{\bPhi}_T(u)=\frac1{\sqrt T}\sum_{t=1}^{\lfloor Tu\rfloor}(\btheta^{(t)}-\btheta_m).
\]
Since
\[
\sup_{u\in[0,1]}\norm{\widetilde{\bPhi}_T(u)-\bPhi_T(u)}=\frac1{\sqrt T}\norm{\btheta^{(0)}-\btheta_m}\to0,
\]
the two processes have the same weak limit. Moreover,
\begin{align*}
\sqrt T(\bar{\btheta}_T-\btheta_m)&=\widetilde{\bPhi}_T(1),\\
\frac1{\sqrt T}\left\{\sum_{k=1}^t\btheta^{(k)}-t\bar{\btheta}_T\right\}&=\widetilde{\bPhi}_T(t/T)-\frac{t}{T}\widetilde{\bPhi}_T(1).
\end{align*}
It follows from Theorem~\ref{thm:syn_QTD_FCLT} and the continuous mapping theorem that
\begin{align*}
    T\widehat{\bV}_T&=\frac{1}{T}\sum_{t=1}^T\left\{\widetilde{\bPhi}_T(t/T)-\frac{t}{T}\widetilde{\bPhi}_T(1)\right\}\left\{\widetilde{\bPhi}_T(t/T)-\frac{t}{T}\widetilde{\bPhi}_T(1)\right\}^{\top}\\
    &\Rightarrow\bG^{-1}\bGamma^{1/2}\left[\int_0^1\prn{\bm B(u)-u\bm B(1)}\prn{\bm B(u)-u\bm B(1)}^\top\,\rd u\right]\prn{\bG^{-1}\bGamma^{1/2}}^\top.
\end{align*}

The convergence in the preceding display holds jointly with $\widetilde{\bPhi}_T(1)$ by the continuous mapping theorem. Let $\sigma_c^2=\bm c^\top\bG^{-1}\bGamma\bG^{-\top}\bm c$. Then $\sigma_c^{-1}\bm c^\top\bG^{-1}\bGamma^{1/2}\bm B(\cdot)$ is a one-dimensional standard Brownian motion. Since the limiting Brownian-bridge integral is positive almost surely, another application of the continuous mapping theorem yields Theorem~\ref{thm:QTD_random_scaling}.

For Corollary~\ref{cor:QTD_random_scaling_CI}, the independence of $B(1)$ and the Brownian bridge implies that $\gV$ has a continuous, symmetric distribution, giving the stated coverage. Moreover,
\[
 T\bm c^\top\widehat{\bV}_T\bm c\Rightarrow \sigma_c^2\int_0^1\{B(u)-uB(1)\}^2\rd u.
\]
The asserted limit for the interval length and its $O_{\PB}(T^{-1/2})$ rate follow immediately.

\subsection{Multiscale Betting}\label{sec:av_proofs}\label{app:anytime}\label{sec:av_width_proofs}\label{app:av_width_proof}
Let $\gO_t=\{(s,A_{t,s},R_{t,s},S'_{t,s}):I_{t,s}=1\}$ and define
\[
\mathscr F_t=\mathscr F_0\vee\sigma(\mathcal O_1,\ldots,\mathcal O_t),
\qquad
\mathscr G_t=\mathscr F_{t-1}\vee\sigma(I_{t,s}:s\in\gS),
\]
where $\mathscr F_0=\sigma(\btheta^{(0)})$, augmented by $s^{(1)}$ under Markovian sampling.
Integrating gives
\[ \EB[Z_{t,s,i}(\btheta)\mid\mathscr G_t]=I_{t,s}h_{s,i}(\btheta),
 \qquad -\tau_i\le Z_{t,s,i}(\btheta)\le1-\tau_i.
\]
Since $\bh(\btheta_m)=\bm0$, each $K_{t,s,i,k}^{\pm}(\btheta_m)$ is a nonnegative martingale. The full mixture
\[
 E_t^\infty(\btheta)\coloneqq\frac1{2d}\sum_{s,i}\sum_{k=0}^{\infty}w_k
 \{K_{t,s,i,k}^{+}(\btheta)+K_{t,s,i,k}^{-}(\btheta)\}
\]
is bounded by $(3/2)^t$, has initial value one, and is a martingale at $\btheta_m$. By Ville's inequality, the event
\begin{equation*}
 \mathcal V_\alpha\coloneqq
 \{E_t^\infty(\btheta_m)<B_\alpha\text{ for every }t\ge0\}
\end{equation*}
has probability at least $1-\alpha/2$. Since $E_t\le E_t^\infty$, it implies $\btheta_m\in C_t$ for every $t$.

For the length bound, define
\[
 N_s(t)=\sum_{u=1}^tI_{u,s},\qquad
 N_{\min}(t)=\min_sN_s(t),\qquad
 \widehat h_{t,s,i}(\btheta)= \frac{\sum_{u=1}^tZ_{u,s,i}(\btheta)}{N_s(t)\vee1},
\]
and write
\begin{equation*}
 \ell_k=\log\frac{4d}{\alpha w_k},\qquad
 b_k=2\sqrt{\frac{\ell_k}{2^k}},\qquad
 b(n)= b_{\lfloor\log_2n\rfloor},\quad n\ge1.
\end{equation*}
The following lemma bounds the empirical Bellman residuals of parameters in $C_t$. Its proof is in Appendix~\ref{app:av_direct_lemmas}.

\begin{lemma}\label{lem:av_residual}
For every $t\ge1$, each $\btheta\in C_t$ satisfies $|\hat h_{t,s,i}(\btheta)|\le b(N_s(t))$ for every $(s,i)\in\gS\times[m]$ with $N_s(t)>0$.
\end{lemma}
The following deterministic comparison lemma, which will also be used for the confidence region constructed from dyadic histograms, converts residual bounds into a bound on the candidate parameters. Its proof is in Appendix~\ref{app:av_direct_lemmas}.

\begin{lemma}\label{lem:max-coordinate}
Fix $t\ge1$ such that $N_{\min}(t)>0$, and let $\mathcal C\subseteq\Theta_m$. For every $(s,i)$, let $a_{s,i}\ge0$ and $\rho_{s,i}>0$. Suppose that
\[
 |\widehat h_{t,s,i}(\btheta)|\le a_{s,i}
 \quad\text{for every }\btheta\in\mathcal C,
\]
and
\[
 \widehat h_{t,s,i}(\btheta_m+\rho_{s,i}\bm1)>a_{s,i},\qquad
 \widehat h_{t,s,i}(\btheta_m-\rho_{s,i}\bm1)<-a_{s,i},
\]
where $\bm1$ is the all-ones vector in $\RB^d$. Then every $\btheta\in\mathcal C$ satisfies $\|\btheta-\btheta_m\|_\infty\le\max_{s,i}\rho_{s,i}$.
\end{lemma}

To obtain the required comparison inequalities, observe that
\begin{equation}\label{eq:av_shift_identity}
 h_{s,i}(\btheta_m+\rho\bm1)
 =F_s^\star(\theta_m(s,i)+(1-\gamma)\rho)-\tau_i,
 \qquad \rho\in\RB.
\end{equation}
We combine the margin in~\eqref{eq:av_cdf_margin} with the following concentration bound, proved in Appendix~\ref{app:av_direct_lemmas}.

\begin{lemma}\label{lem:av_visit_boundary}
Fix $\bv\in\RB^d$, $(s,i)\in\gS\times[m]$, an integer $n_0\ge1$, and $x>0$. Under each sampling mechanism, for $\sigma\in\{-1,1\}$,
\begin{equation*}
 \PB\left(\exists t:\ N_s(t)\ge n_0,\
 \sigma\{\widehat h_{t,s,i}(\bv)-h_{s,i}(\bv)\}\ge x\right)
 \le e^{-2n_0x^2}.
\end{equation*}
\end{lemma}

It is easy to verify that $b_k$ decreases to zero. Choose a deterministic $k_0$ such that $2b_k\le\kappa_mr_m$ for $k\ge k_0$, and set
\[
 \rho_k\coloneqq\frac{2b_k}{(1-\gamma)\kappa_m},\qquad
 \bv_k^\pm\coloneqq\btheta_m\pm\rho_k\bm1.
\]
Equations~\eqref{eq:av_cdf_margin} and~\eqref{eq:av_shift_identity} give $h_{s,i}(\bv_k^+)\ge2b_k$ and $h_{s,i}(\bv_k^-)\le-2b_k$. Apply Lemma~\ref{lem:av_visit_boundary} with $n_0=2^k$ and $x=b_k$ to the unfavorable tail at each comparison vector. Since
\begin{equation}\label{eq:av_barrier_failure}
 2d\sum_{k\ge k_0}e^{-2\cdot2^kb_k^2}
 =2d\sum_{k\ge k_0}e^{-8\ell_k}
 \le2d\sum_{k\ge0}\frac{\alpha w_k}{4d}=\frac\alpha2,
\end{equation}
there is an event $\mathcal H_\alpha$ of probability at least $1-\alpha/2$ on which, simultaneously for every $(s,i)$, $k\ge k_0$, and $t$ with $N_s(t)\ge2^k$,
\begin{equation}\label{eq:av_barriers}
 \widehat h_{t,s,i}(\bv_k^+)>b_k,\qquad
 \widehat h_{t,s,i}(\bv_k^-)<-b_k.
\end{equation}
For $N_{\min}(t)\ge2^{k_0}$, put $k_s=\lfloor\log_2N_s(t)\rfloor$. On $\mathcal H_\alpha$, apply Lemma~\ref{lem:max-coordinate} with $\mathcal C=C_t$, $a_{s,i}=b_{k_s}$, and $\rho_{s,i}=\rho_{k_s}$. By Lemma~\ref{lem:av_residual} and~\eqref{eq:av_barriers},
\[
 \norm{\btheta-\btheta_m}_\infty
 \le\frac{2b(N_{\min}(t))}{(1-\gamma)\kappa_m},
 \qquad\btheta\in C_t.
\]
Taking the range of $\bm c^\top\btheta$ therefore gives
\begin{equation}\label{eq:av_width_bound}
 \operatorname{length}(I_t^c)
 \le\frac{4\norm{\bm c}_1b(N_{\min}(t))}{(1-\gamma)\kappa_m}.
\end{equation}
The sampling assumptions imply
\begin{equation}\label{eq:av_visit_frequencies}
 N_s(t)/t\to\nu_s>0\quad\text{almost surely},
 \qquad\nu_{\min}\coloneqq\min_s\nu_s,
\end{equation}
where $\nu_s=1$, $\mu(s)$, or $d^\pi(s)$ in the synchronous, asynchronous, or Markovian setting. The definition of $b(n)$ gives
\begin{equation}\label{eq:av_boundary_order}
 b(n)\le5\sqrt{\frac{\log\log n}{n}}
 \quad\text{for all }n\geq n_0,
\end{equation}
where $n_0$ may depend on $\alpha,d$. Therefore $N_{\min}(t)\ge\nu_{\min}t/2$ eventually, so~\eqref{eq:av_width_bound} proves~\eqref{eq:av_global_rate} with
\[ C=\frac{20}{(1-\gamma)\kappa_m}\sqrt{\frac{2}{\nu_{\min}}}.
\]
On the intersection of $\gV_\alpha$, $\gH_\alpha$, and the probability-one event in~\eqref{eq:av_visit_frequencies}, both conclusions of Theorem~\ref{thm:av_width} hold. This intersection has probability at least $1-\alpha$.

\subsection{Dyadic Histogram Inversion}\label{sec:dh_proofs}
For a reward in $B_{j,\ell}$, the score bounds satisfy $g^-\le g\le g^+$. Each histogram product is therefore bounded by its counterpart in~\eqref{eq:av_multiscale_wealth}, showing that for every $t\geq 1$,
\[ \underline E_t(\btheta)\le E_t(\btheta),\qquad C_t\subseteq\widetilde C_t.
\]
On $\mathcal V_\alpha$, we therefore have $\btheta_m\in\widetilde C_t$ for every $t$. This event has probability at least $1-\alpha/2$.

To control the interval length, let $g_{u,s,i}^{\pm}(\btheta)$ denote the score bounds defined in Section~\ref{subsec:dyadic_inference} for the transition observed at round $u$ from state $s$. Define
\[
 A_{t,s}=\sup_{\btheta\in\Theta_m}\max_i
 \sum_{u\le t:\,I_{u,s}=1}
 [g^+_{u,s,i}(\btheta)-g^-_{u,s,i}(\btheta)],
 \qquad A_t=\max_s A_{t,s}.
\]
The following lemma bounds the additional empirical residual caused by the histograms. Its proof is in Appendix~\ref{app:dh_lemmas}.

\begin{lemma}\label{lem:dh_histogram}
Fix $\zeta\in[1/2,1)$. Under each sampling mechanism, there exists a constant $C_A>0$, depending only on $C_0$, $|\gS|$, and $\zeta$, such that $A_t\le C_A t^{1-\zeta}$ almost surely for all sufficiently large $t$. Moreover, for every $t\ge1$, $\btheta\in\widetilde C_t$, and $(s,i)$ with $N_s(t)>0$,
\begin{equation}\label{eq:dh_residual}
|\widehat h_{t,s,i}(\btheta)|
\le b(N_s(t))+\frac{A_{t,s}}{N_s(t)}.
\end{equation}
\end{lemma}
On the probability-one event supplied by this lemma and~\eqref{eq:av_visit_frequencies}, eventually
\begin{equation}\label{eq:dh_ambiguity_rate}
 \frac{A_{t,s}}{N_s(t)}\le Dt^{-\zeta}\quad\text{for every }s,
 \qquad D\coloneqq\frac{2C_A}{\nu_{\min}}.
\end{equation}
For each deterministic scale $k$, set
\[ a_k= D2^{-\zeta k},\qquad
 \widetilde\rho_k=\frac{2b_k+a_k}{(1-\gamma)\kappa_m},\qquad
 \widetilde{\bv}_k^\pm=\btheta_m\pm\widetilde\rho_k\bm1.
\]
Choose a deterministic $k_1$ such that $2b_k+a_k\le\kappa_mr_m$ for $k\ge k_1$. The margin~\eqref{eq:av_cdf_margin} and shift identity~\eqref{eq:av_shift_identity} give
\[
 h_{s,i}(\widetilde{\bv}_k^+)\ge2b_k+a_k,\qquad
 h_{s,i}(\widetilde{\bv}_k^-)\le-2b_k-a_k.
\]
Apply Lemma~\ref{lem:av_visit_boundary} to these nonrandom vectors with $n_0=2^k$ and $x=b_k$. The same summation as in~\eqref{eq:av_barrier_failure} gives an event $\widetilde{\mathcal H}_\alpha$ of probability at least $1-\alpha/2$ on which, for every $(s,i)$, $k\ge k_1$, and $t$ with $N_s(t)\ge2^k$,
\begin{equation}\label{eq:dh_barriers}
 \widehat h_{t,s,i}(\widetilde{\bv}_k^+)>b_k+a_k,\qquad
 \widehat h_{t,s,i}(\widetilde{\bv}_k^-)<-b_k-a_k.
\end{equation}

For all sufficiently large $t$, put $k_s=\lfloor\log_2N_s(t)\rfloor\ge k_1$. Since $2^{k_s}\le N_s(t)\le t$,~\eqref{eq:dh_ambiguity_rate} implies $A_{t,s}/N_s(t)\le a_{k_s}$. 
Apply Lemma~\ref{lem:max-coordinate} to $\mathcal C=\widetilde C_t$, with thresholds $b_{k_s}+a_{k_s}$ and radii $\widetilde\rho_{k_s}$. Combining~\eqref{eq:dh_residual} and~\eqref{eq:dh_barriers} with the monotonicity of $b_k$ and $a_k$ yields that for every $\btheta\in\widetilde C_t$,
\[
 \norm{\btheta-\btheta_m}_\infty
 \le\frac{2b(N_{\min}(t))+D2^{-\zeta\lfloor\log_2N_{\min}(t)\rfloor}}
 {(1-\gamma)\kappa_m}.
\]
Eventually $N_{\min}(t)\ge\nu_{\min}t/2$ and $2^{\lfloor\log_2N_{\min}(t)\rfloor}>\nu_{\min}t/4$. By~\eqref{eq:av_boundary_order}, taking the range of $\bm c^\top\btheta$ proves~\eqref{eq:dh_length} with
\[
 C_1\coloneqq\frac{20}{(1-\gamma)\kappa_m}\sqrt{\frac2{\nu_{\min}}},\qquad
 C_2\coloneqq\frac{2D}{(1-\gamma)\kappa_m}
             \left(\frac4{\nu_{\min}}\right)^\zeta.
\]
On the intersection of $\mathcal V_\alpha$, $\widetilde{\mathcal H}_\alpha$, and the probability-one event in~\eqref{eq:dh_ambiguity_rate}, the conclusions of Theorem~\ref{thm:dh_equivalence} hold. This intersection has probability at least $1-\alpha$.

\subsection{Iterated-Logarithm Lower Bounds}\label{sec:lil_proof}\label{app:lil}

Throughout this subsection and the corresponding proofs, let
$\rho_n=\sqrt{2\log\log n/n}$ for $n\ge3$.
The following two lemmas are proved in Appendix~\ref{app:lil_lemma_proofs}.
The first lemma transfers the lower bound for sequential Gaussian mean estimation to confidence intervals in experiments whose observations can be generated from Gaussian samples.

\begin{lemma}\label{lem:lil_transfer}
Let $J$ be a compact interval of positive length. Suppose that, for each $\beta\in J$, the observations $(Y_t)_{t\ge1}$ can be generated jointly as
\begin{equation}\label{eq:lil_simulator}
 (Y_1,\ldots,Y_t)=F_t(X_1,\ldots,X_t,U),
 \qquad X_i\simiid\normal(\beta,1),\quad t\ge1,
\end{equation}
where the maps $F_t$ do not depend on $\beta$, and $U$ is independent of $(X_i)_{i\ge1}$ with a distribution independent of $\beta$.

Let $\psi$ be continuously differentiable on a neighborhood of $J$ with $\min_{\beta\in J}|\psi'(\beta)|>0$. For any interval rule $(I_t)_{t\ge1}$ based on observations through time $t$, possibly using independent randomization, the following holds for Lebesgue-almost every $\beta\in J$:
\begin{equation*}
 \PB_\beta\left(
 \{\psi(\beta)\in I_t\ \forall t\}
 \cap\left\{\limsup_{t\to\infty}
          \frac{\operatorname{length}(I_t)}{\rho_t}
                   <2|\psi'(\beta)|\right\}\right)=0.
\end{equation*}
\end{lemma}

The second lemma constructs a regular MDP family with nonzero sensitivity of the target contrast to the Gaussian mean, while allowing exact simulation using at most one Gaussian sample per round under all three sampling mechanisms.
\begin{lemma}\label{lem:lil_qtd_submodel}
Fix $\gS,\gA,\pi,m,\gamma,\bm c$ as in Theorem~\ref{thm:lil_lower}. There exist $b>0$ and a family of MDPs $\{\gM_\beta:\beta\in J\coloneqq[-b,b]\}$ satisfying Assumptions~\ref{assump:density}--\ref{assump:recurrence}, with transition kernel $P(s'\mid s,a)=1/|\gS|$ and uniformly bounded reward densities and Lipschitz constants. The target contrast
$\psi(\beta)\coloneqq\bm c^\top\btheta_m(\gM_\beta)$ admits a continuously differentiable extension to a neighborhood of $J$, with
$a_J\coloneqq\min_{\beta\in J}|\psi'(\beta)|>0$.

For each sampling mechanism in Theorem~\ref{thm:lil_lower}, let $Y_t$ denote the data collected at round $t$ under $\gM_\beta$. The process $(Y_t)_{t\ge1}$ can be generated jointly as in~\eqref{eq:lil_simulator}, using at most one new Gaussian sample per round.
\end{lemma}
The construction perturbs the reward distribution at a single state $s_0$, giving a density of the form $p_\beta=\Phi(\beta)p_++\{1-\Phi(\beta)\}p_-$, where $\Phi$ is the standard normal CDF and $p_+,p_-$ are fixed densities. 
Since $\PB_\beta(X_t\ge0)=\Phi(\beta)$ for $X_t\sim\normal(\beta,1)$, the mixture representation allows us to generate a reward at $s_0$ by choosing $p_+$ or $p_-$ according to the sign of $X_t$. The remaining observations are generated using auxiliary randomness. Since each round contains at most one reward from $s_0$, this reproduces the joint observation law under all three sampling mechanisms. The reward perturbation is chosen using the linearized Bellman equation to ensure $\psi'(0)\ne0$, and the derivative remains nonzero on a sufficiently small $J$.

To prove Theorem~\ref{thm:lil_lower}, take the family from Lemma~\ref{lem:lil_qtd_submodel}. For any fixed interval rule, Lemma~\ref{lem:lil_transfer} gives~\eqref{eq:lil_on_coverage} for Lebesgue-almost every $\beta\in J$. For those $\beta$, we have
\[
 \PB_\beta\brk{A_\beta\setminus \brc{\limsup_{t\to\infty}
 \operatorname{length}(I_t)/\rho_t\ge2|\psi'(\beta)|}}=0.
\]
Since $\PB_\beta(A_\beta)\ge1-\alpha$ and $|\psi'(\beta)|\ge a_J$, we obtain Equation~\eqref{eq:lil_probability_lower} with $C=2a_J$.

%% file: numerical_experiments.tex
We evaluate the two inference procedures in two simulated environments, RiverSwim and a finite model calibrated from Simglucose, and in an application to Belgian electricity imbalance settlement data. In the simulations, we compare random scaling (RS) with online batch means (OBM) on the same trajectories in terms of empirical coverage and interval length, and examine the coverage and length of dyadic histogram confidence intervals. For the electricity data, we construct confidence intervals for two risk functionals using a single observation stream.

In each simulated environment, we run $1000$ independent replications of $500000$ sampling rounds per mechanism, evaluating the confidence intervals at ten equally spaced checkpoints. We use a nominal confidence level of $95\%$ throughout and evaluate empirical coverage using reference fixed points computed by quantile dynamic programming.

For RS and OBM, QTD is initialized at zero, with step-size parameters selected through preliminary simulations. The choices for RiverSwim and Simglucose are reported in Table~\ref{tab:experiment_steps}, while the electricity application uses $(a,c)=(0.505,25)$. For a prescribed contrast $\bv$, RS uses Corollary~\ref{cor:QTD_random_scaling_CI} and the Brownian critical value $q_{0.025}=6.7483$, estimated by Monte Carlo simulation of the limiting Brownian functional $\gV$. For OBM, let $X_t=\bv^\top\btheta^{(t)}$ and $\bar X_T=T^{-1}\sum_{t=1}^T X_t$. Following~\citet{zhu2023onlinecovariance}, use block boundaries $b_k=\lfloor k^{2/(1-a)}\rfloor$ and, for $b_k\le i<b_{k+1}$, define $t_i=b_k$, $\ell_i=i-t_i+1$, and $W_i=\sum_{j=t_i}^i X_j$. The variance estimator and confidence interval are
\[
\widehat\sigma^2_{\mathrm{OBM},T}
=\frac{\sum_{i=1}^T(W_i-\ell_i\bar X_T)^2}{\sum_{i=1}^T\ell_i},
\qquad
I_T^{\mathrm{OBM}}
=\left[\bar X_T\pm1.96\sqrt{\widehat\sigma^2_{\mathrm{OBM},T}/T}\right].
\]

For Dyadic, we use $\zeta=0.65$ and $B_\alpha=20$ in all three environments. Under the stated model assumptions, this threshold guarantees at least $95\%$ simultaneous coverage. At each checkpoint, we numerically compute outer bounds for the projection of the histogram-based confidence region.

\begin{table}[tbp]
\centering
\small
\setlength{\tabcolsep}{6pt}
\begin{tabular}{lccc}
\toprule
Environment & Synchronous $(a,c)$ & Asynchronous $(a,c)$ & Markovian $(a,c)$ \\
\midrule
RiverSwim & $(0.505,0.60)$ & $(0.530,2.65)$ & $(0.545,1.88)$ \\
Simglucose & $(0.55,4)$ & $(0.55,24)$ & $(0.55,64)$ \\
\bottomrule
\end{tabular}
\caption{QTD step-size parameters used in the simulation studies.}
\label{tab:experiment_steps}
\end{table}

\subsection{RiverSwim}
\label{subsec:experiment_riverswim}

Our first environment is a RiverSwim MDP with six states $\gS=\{0,\ldots,5\}$, two actions $\gA=\{L,R\}$, and discount factor $\gamma=0.8$. The policy chooses $L$ with probability $0.3$ and $R$ with probability $0.7$. For $s\ge1$, action $L$ moves to $s-1$ with probability $0.9$, while for $s\le4$, action $R$ moves to $s+1$ with probability $0.7$. All remaining transition probability is assigned to staying at $s$. The reward distributions are
\[
R\mid(s,a)\sim
\begin{cases}
\operatorname{Beta}(5,1), & (s,a)=(5,R),\\
\operatorname{Beta}(1,15), & \text{otherwise}.
\end{cases}
\]
We take $m=20$ and estimate $\theta_m(2,10)$, corresponding to the quantile level $\tau_{10}=0.475$. Asynchronous sampling selects states uniformly, and Markovian trajectories start at state $0$.

Figure~\ref{fig:experiment_riverswim_obm} and Table~\ref{tab:experiment_final} compare RS with OBM at the recorded checkpoints. RS achieves final coverage of $96.4\%$, $93.3\%$, and $93.6\%$ under synchronous, asynchronous, and Markovian sampling, compared with $82.3\%$, $76.7\%$, and $59.1\%$ for OBM. Across the three sampling mechanisms, RS coverage remains relatively close to the nominal level throughout. Although OBM produces shorter intervals, its coverage improves only gradually and remains substantially below the nominal level.

\begin{figure}[tbp]
\centering
\includegraphics[width=\textwidth]{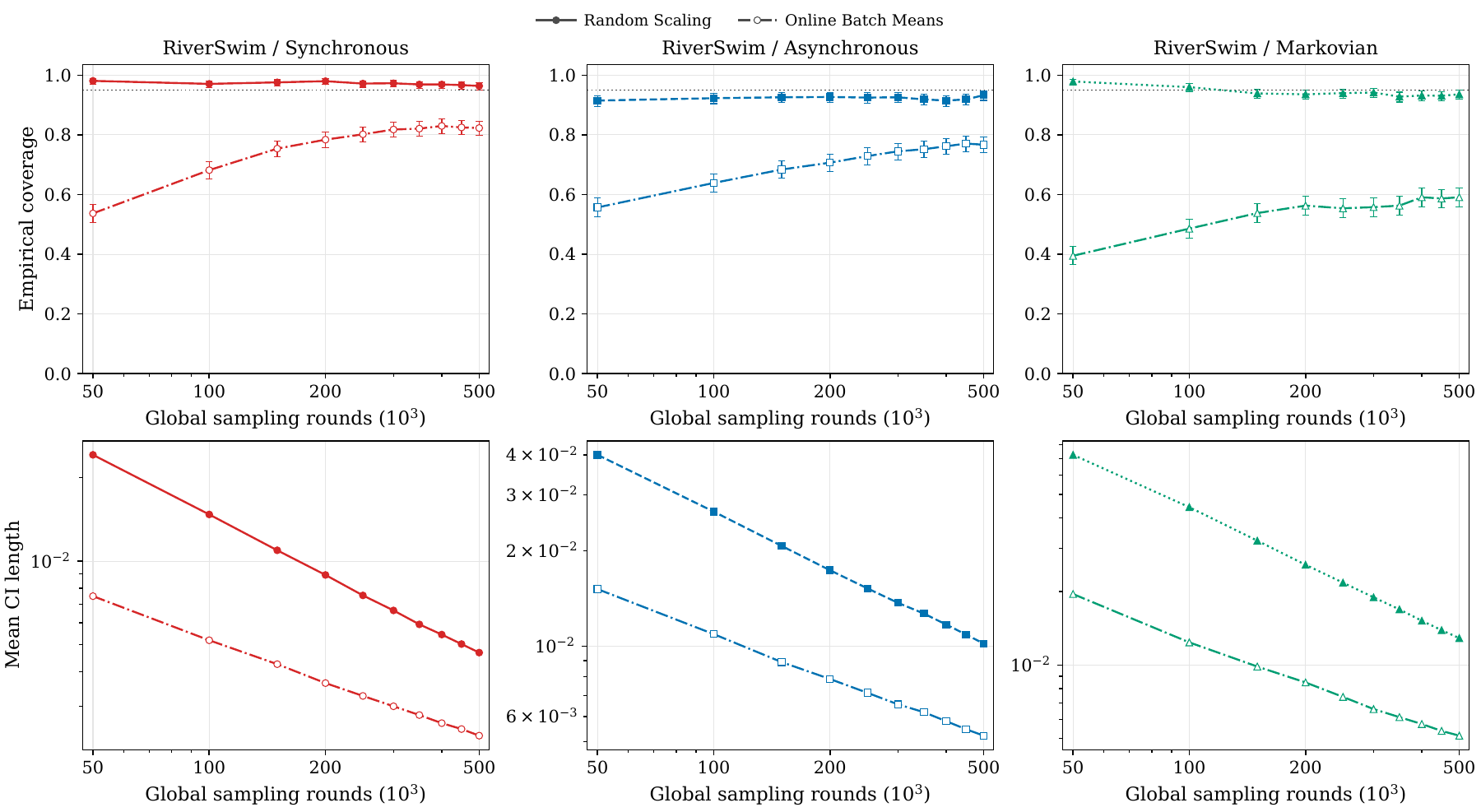}
\caption{RS and OBM coverage (top) and mean interval length (bottom) under the three sampling mechanisms in RiverSwim. Error bars show $95\%$ Clopper--Pearson intervals for empirical coverage, and the horizontal line marks the nominal $95\%$ level.}
\label{fig:experiment_riverswim_obm}
\end{figure}

For Dyadic, every interval contains the reference coordinate at all ten checkpoints in each of the $1000$ replications per sampling mechanism. This perfect empirical coverage may reflect the conservativeness of the time-uniform guarantee based on Ville's inequality, together with the additional enlargement from histogram approximation. Figure~\ref{fig:experiment_dyadic}(a) shows that the mean interval lengths decrease steadily with the number of sampling rounds, but remain substantially larger than those of RS at the final checkpoint (Table~\ref{tab:experiment_final}).

\subsection{Simglucose}
\label{subsec:experiment_simglucose}

The second environment is a 15-state MDP calibrated from Simglucose trajectories under a fixed basal--bolus controller. Each transition represents one hour. The states summarize glucose levels, glucose trends, insulin and carbohydrate information, and time of day. Rewards are constructed from hourly glucose risk using a smoothed Beta-mixture distribution. Details of the data generation, state construction, and reward model are provided in Appendix~\ref{app:simglucose_model}.

We set $\gamma=0.9$ and $m=10$, giving a parameter dimension of $150$. The target is $\theta_m(0,1)$, corresponding to the $0.05$-quantile in the severe hypoglycemia state (glucose below $54$ mg/dL). Asynchronous sampling selects states uniformly, while Markovian trajectories start from the stationary distribution.

Figure~\ref{fig:experiment_simglucose_obm} shows that RS reaches final coverage of $95.0\%$ and $94.9\%$ in the synchronous and asynchronous settings, compared with $87.9\%$ and $74.0\%$ for OBM. Under Markovian sampling, RS coverage improves from $84.3\%$ at the first checkpoint to $98.2\%$ at the last, while OBM rises from $2.5\%$ to only $23.9\%$. OBM again produces shorter intervals but substantially lower coverage. The target state has stationary probability only $0.00551$, so infrequent visits may prolong the initialization transient of QTD and affect the finite-sample performance of both procedures.

\begin{figure}[tbp]
\centering
\includegraphics[width=\textwidth]{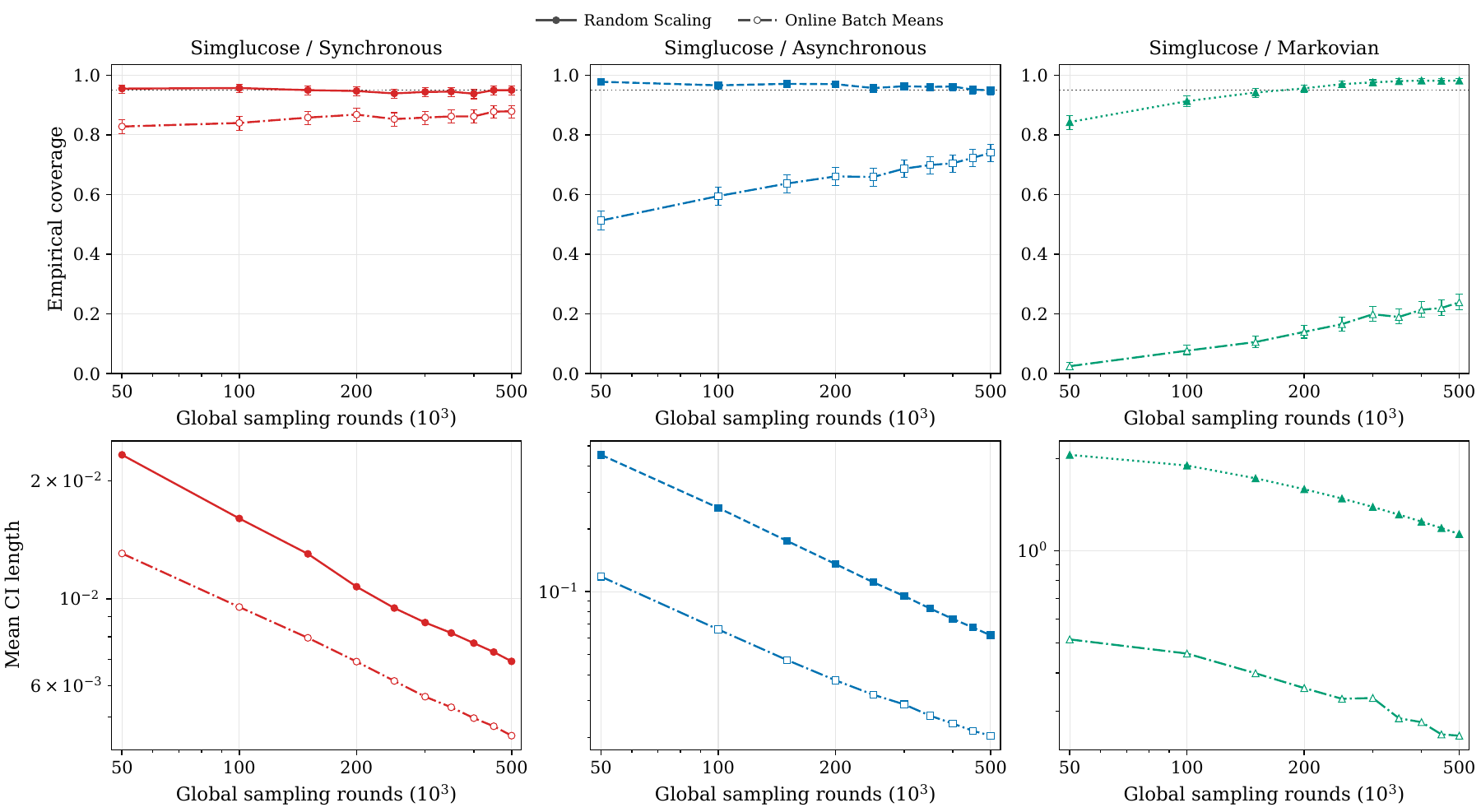}
\caption{RS and OBM coverage (top) and mean interval length (bottom) for the severe hypoglycemia target in Simglucose. Plotting conventions are as in Figure~\ref{fig:experiment_riverswim_obm}.}
\label{fig:experiment_simglucose_obm}
\end{figure}

As shown in Table~\ref{tab:experiment_final} and Figure~\ref{fig:experiment_dyadic}(b), Dyadic achieves full empirical coverage at all ten checkpoints across the $1000$ replications under each sampling mechanism, and the mean interval lengths decrease steadily under all three mechanisms. Under Markovian sampling, Dyadic even produces shorter intervals than RS at the final checkpoint, despite providing a stronger time-uniform coverage guarantee. This may partly reflect the effect of persistent initialization transients on the QTD-based random-scaling procedure. 

Across both simulated environments, RS achieves substantially higher empirical coverage than OBM, while Dyadic produces steadily shrinking intervals that remain competitive in length with RS, despite its stronger time-uniform coverage guarantee.

\begin{figure}[!htb]
\centering
\includegraphics[width=0.9\textwidth]{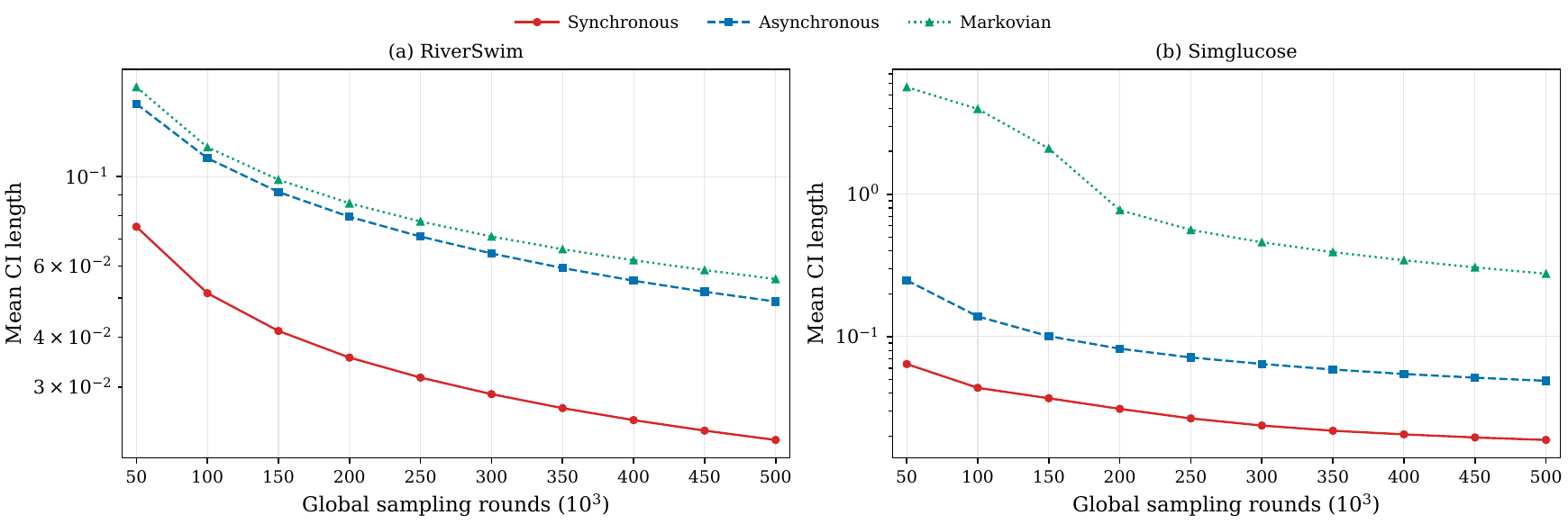}
\caption{Mean lengths of Dyadic outer confidence intervals in (a) RiverSwim and (b) Simglucose under synchronous, asynchronous, and Markovian sampling.}
\label{fig:experiment_dyadic}
\end{figure}

\begin{table}[tbp]
\centering
\small
\setlength{\tabcolsep}{4.5pt}
\begin{tabular}{llrrr rrr}
\toprule
 & & \multicolumn{3}{c}{Coverage (\%)} & \multicolumn{3}{c}{Mean interval length} \\
\cmidrule(lr){3-5}\cmidrule(lr){6-8}
Environment & Sampling & RS & OBM & Dyadic & RS & OBM & Dyadic \\
\midrule
RiverSwim & Synchronous  & 96.4 & 82.3 & 100.0 & 0.00468 & 0.00235 & 0.02219 \\
          & Asynchronous & 93.3 & 76.7 & 100.0 & 0.01018 & 0.00522 & 0.04899 \\
          & Markovian    & 93.6 & 59.1 & 100.0 & 0.01289 & 0.00514 & 0.05570 \\
\addlinespace
Simglucose & Synchronous  & 95.0 & 87.9 & 100.0 & 0.00692 & 0.00447 & 0.01878 \\
           & Asynchronous & 94.9 & 74.0 & 100.0 & 0.06185 & 0.02042 & 0.04888 \\
           & Markovian    & 98.2 & 23.9 & 100.0 & 1.13428 & 0.24972 & 0.27638 \\
\bottomrule
\end{tabular}
\caption{Empirical coverage (\%) and mean interval length at $T=500000$, based on $1000$ replications per setting. Dyadic results are computed from the reported outer confidence intervals.}
\label{tab:experiment_final}
\end{table}


\subsection{Electricity Imbalance Settlement}
\label{subsec:experiment_elia}

We apply RS and Dyadic to Elia's ODS047 data on Belgian electricity imbalance settlement.\footnote{Elia Open Data, ``Imbalance prices per quarter-hour (Historical data - up to 22/05/2024),'' dataset ODS047: \url{https://opendata.elia.be/explore/dataset/ods047/}.} We use $175392$ quarter-hourly settlement records from 2016 to 2020. Negative settlement prices are unfavorable for a positive imbalance position, making the lower tail of cumulative payoffs particularly relevant for risk assessment. We quantify this downside risk through return quantiles and conditional value-at-risk.

We construct a seven-state Markov reward process (MRP) from system imbalance and net regulation volume. Three states represent shortage with increasing regulation magnitude, one represents balance, and three represent surplus with increasing regulation magnitude. Consecutive settlement records determine transitions, while settlement prices are transformed into bounded rewards through an increasing function. We set $\gamma=0.9$ and $m=20$. Details of the construction are given in Appendix~\ref{app:elia_model}.

We consider two measures of downside risk for the projected return distribution. The first is the $10\%$ value-at-risk (VaR) in state $S2$, corresponding to shortage with high regulation magnitude. The second is the lower $10\%$ conditional value-at-risk (CVaR) in state $S6$, corresponding to surplus with high regulation magnitude. These functionals take the forms
\begin{align*}
\operatorname{VaR}_{0.10}\bigl(\eta_{20}(S2)\bigr)
&=\theta_{20}(S2,2),\\
\operatorname{CVaR}^{\mathrm{lower}}_{0.10}\bigl(\eta_{20}(S6)\bigr)
&=\frac{1}{0.10}\int_0^{0.10}F^{-1}_{\eta_{20}(S6)}(u)\,\rd u
=\frac{\theta_{20}(S6,1)+\theta_{20}(S6,2)}{2}.
\end{align*}

\begin{figure}[!htb]
\centering
\includegraphics[width=\textwidth]{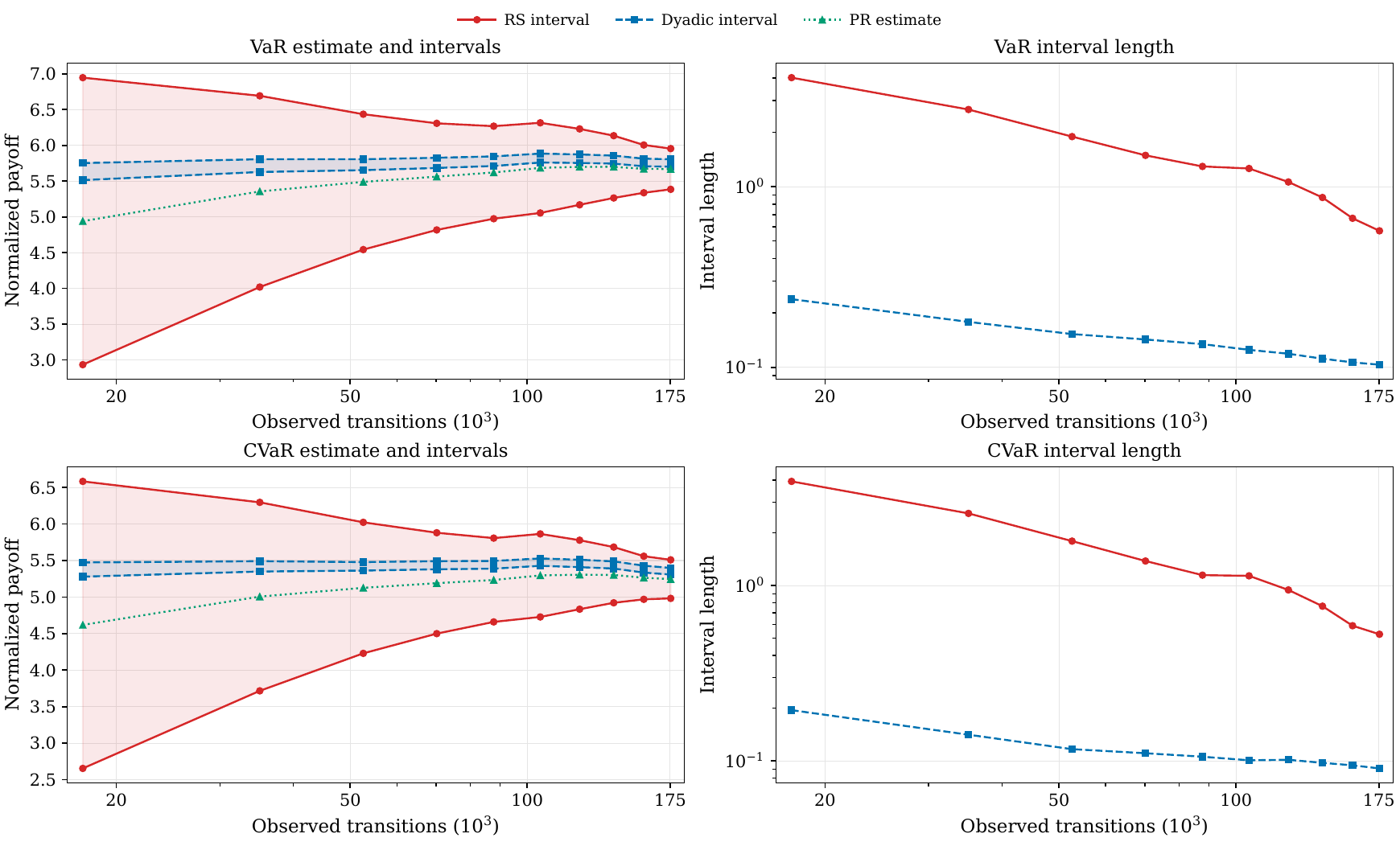}
\caption{Results on Elia settlement data, 2016--2020. The left panels show RS intervals, PR estimates, and Dyadic outer intervals for the projected return VaR in state $S2$ and lower CVaR in state $S6$. The right panels show interval lengths on a logarithmic scale.}
\label{fig:experiment_elia}
\end{figure}

\begin{table}[tbp]
\centering
\small
\setlength{\tabcolsep}{7pt}
\begin{tabular}{llrrr}
\toprule
Target & Method & Lower endpoint & Upper endpoint & Length \\
\midrule
$S2$: $10\%$ VaR & RS     & 5.38553 & 5.95512 & 0.56959 \\
                 & Dyadic & 5.70231 & 5.80571 & 0.10340 \\
\addlinespace
$S6$: lower $10\%$ CVaR & RS     & 4.98174 & 5.50962 & 0.52788 \\
                        & Dyadic & 5.30814 & 5.39878 & 0.09064 \\
\bottomrule
\end{tabular}
\caption{Final intervals for the Elia application after $175391$ observed transitions.}
\label{tab:experiment_elia}
\end{table}

We evaluate the intervals at ten equally spaced checkpoints along the observed sequence. Figure~\ref{fig:experiment_elia} shows that the interval lengths of both procedures decrease overall as more observations become available, with Dyadic producing shorter intervals for both risk functionals. The Polyak--Ruppert estimates gradually approach the Dyadic intervals, and the intervals from the two procedures overlap for both targets at the final checkpoint (Table~\ref{tab:experiment_elia}). These results show that both procedures can provide informative confidence intervals for different risk functionals from a single stream of settlement observations.

%% file: conclusion.tex
In this paper, we have studied online statistical inference for quantile distributional policy evaluation under synchronous, asynchronous, and Markovian sampling. We have established functional central limit theorems for QTD and developed a random-scaling procedure that provides asymptotically exact confidence intervals at a fixed time without estimating the asymptotic covariance. We have also constructed anytime-valid confidence intervals by inverting multiscale Bellman-score tests, with dyadic reward histograms allowing evidence evaluation from a streaming summary of the observations.

These two approaches address different requirements of online inference. Random scaling yields confidence intervals with lengths of order $O_\PB(T^{-1/2})$, while the anytime-valid intervals maintain simultaneous coverage and attain the classwise optimal iterated-logarithm rate for interval length. Together, our results provide online statistical inference procedures for the quantile Bellman fixed point both through QTD trajectories and directly from sequential observations.

Several directions remain for future research. One is to characterize the sharp leading constants for the lengths of anytime-valid confidence intervals, beyond the classwise optimal rate established in this paper. Another is to extend the inference procedures beyond tabular policy evaluation, particularly to settings with function approximation and continuous state spaces.

%% file: Appendix_omitted_proofs_4.tex
In this appendix, we prove the supporting results stated in Section~\ref{Section:proof_outline}. Technical lemmas used below are stated and proved in Appendix~\ref{Appendix_technical_lemmas}.

\subsection{Proofs of Theorem~\ref{thm:general_FCLT} and Lemma~\ref{lem:consistency}}\label{app:independent_fclt}\label{app:qtd_consistency}
\begin{proof}[Proof of Theorem~\ref{thm:general_FCLT}]
    Denote $\bDelta_t=\bZ_t-\bz^*$ and $\brho_t=\bh(\bZ_t)-\bM(\bZ_t-\bz^*)$, then
    \begin{equation*}
        \bDelta_t=(\bI-\alpha_{t-1}\bM)\bDelta_{t-1}-\alpha_{t-1}(\beps_{t-1}+\brho_{t-1}). 
    \end{equation*}
    Therefore, 
    \begin{equation*}
        \bDelta_t=\prod_{i=0}^{t-1}(\bI-\alpha_i\bM)\bDelta_0-\sum_{i=0}^{t-1}\alpha_i\prod_{l=i+1}^{t-1}(\bI-\alpha_l\bM)(\beps_i+\brho_i). 
    \end{equation*}
    Define
    \begin{equation*}
        \bA_i^t=\alpha_i\sum_{j=i}^{t-1}\prod_{l=i}^{j-1}(\bI-\alpha_l\bM), 
    \end{equation*}
    then
    \begin{align*}
        \frac{1}{\sqrt{T}}\sum_{t=0}^{\lfloor Tu\rfloor}\bDelta_t=&T^{-\frac12}\alpha_0^{-1}\bA_0^{\down{Tu}+1}\bDelta_0\\
        &-T^{-\frac12}\sum_{i=0}^{\down{Tu}}\frac{\alpha_i}{\alpha_{i+1}}\bA_{i+1}^{\down{Tu}+1}\brho_i\\
        &-T^{-\frac12}\sum_{i=0}^{\down{Tu}}\frac{\alpha_i}{\alpha_{i+1}}\prn{\bA_{i+1}^{\down{Tu}+1}-\bA_{i+1}^{T+1}}\beps_i\\
        &-T^{-\frac12}\sum_{i=0}^{\down{Tu}}\prn{\frac{\alpha_i}{\alpha_{i+1}}\bA_{i+1}^{T+1}-\bM^{-1}}\beps_i\\
        &-T^{-\frac12}\sum_{i=0}^{\down{Tu}}\bM^{-1}\beps_i\\
        \coloneqq&I_0+I_1+I_2+I_3+I_4
    \end{align*}
    By~\citet[Lemma~1]{polyak1992acceleration}, there exists a constant $C$ such that $\|\bA_i^t\|\leq C$ for all $t,i$. Therefore, 
    \begin{equation*}
        \sup_{u\in[0,1]}\norm{I_0}\leq T^{-\frac12}C\alpha_0^{-1}\norm{\bDelta_0}\to 0
    \end{equation*}
    almost surely. For $I_1$, we know that
    \begin{equation*}
        \sup_{u\in[0,1]}\norm{I_1}\lesssim T^{-\frac12}\sum_{i=0}^{T}\norm{\brho_i}\to 0
    \end{equation*}
    almost surely by Lemma~\ref{lem:second_order_remainder} and Kronecker's lemma. For $I_2$, we have
    \begin{align*}
        &\norm{T^{-\frac12}\sum_{i=0}^{\down{Tu}}\frac{\alpha_i}{\alpha_{i+1}}\prn{\bA_{i+1}^{\down{Tu}+1}-\bA_{i+1}^{T+1}}\beps_i}\\
        =&\norm{T^{-\frac12}\sum_{i=0}^{\down{Tu}}\alpha_i\sum_{j=\down{Tu}+1}^{T}\prod_{l=i+1}^{j-1}(\bI-\alpha_l\bM)\beps_i}\\
        =&\norm{T^{-\frac12}\bA_{\down{Tu}+1}^{T+1}\alpha_{\down{Tu}+1}^{-1}\sum_{i=0}^{\down{Tu}}\alpha_i\prod_{l=i+1}^{\down{Tu}}(\bI-\alpha_l\bM)\beps_i}\\
        \lesssim&T^{-\frac12}\alpha_{\down{Tu}+1}^{-1}\norm{\sum_{i=0}^{\down{Tu}}\alpha_i\prod_{l=i+1}^{\down{Tu}}(\bI-\alpha_l\bM)\beps_i}. 
    \end{align*}
    By~\citet[Lemma~4]{li2023online}, 
    \begin{equation*}
        \sup_{u\in[0,1]}\norm{I_2}\lesssim\sup_{u\in[0,1]}T^{-\frac12}\alpha_{\down{Tu}+1}^{-1}\norm{\sum_{i=0}^{\down{Tu}}\alpha_i\prod_{l=i+1}^{\down{Tu}}(\bI-\alpha_l\bM)\beps_i}\to0
    \end{equation*}
    in probability. Moreover, by~\citet[Lemma~1]{polyak1992acceleration}, 
    \begin{equation*}
        \lim_{T\to\infty}\frac1T\sum_{i=0}^{T}\norm{\bA_{i+1}^{T+1}-\bM^{-1}}=0. 
    \end{equation*}
    Therefore, by Doob's maximal inequality, 
    \begin{align*}
        &\EB\brk{\sup_{u\in[0,1]}\norm{I_3}^2}\\
        \lesssim&\EB\norm{T^{-\frac12}\sum_{i=0}^{T}\prn{\frac{\alpha_i}{\alpha_{i+1}}\bA_{i+1}^{T+1}-\bM^{-1}}\beps_i}^2\\
        =&T^{-1}\sum_{i=0}^{T}\EB\norm{\prn{\frac{\alpha_i}{\alpha_{i+1}}\bA_{i+1}^{T+1}-\bM^{-1}}\beps_i}^2\\
        \lesssim&T^{-1}\sum_{i=0}^{T}\brk{\norm{\bA_{i+1}^{T+1}-\bM^{-1}}^2+\prn{\frac{\alpha_i}{\alpha_{i+1}}-1}^2}\\
        \lesssim&T^{-1}\sum_{i=0}^{T}\brk{\norm{\bA_{i+1}^{T+1}-\bM^{-1}}+(i+1)^{-2}}\to0. 
    \end{align*}
    The bounded martingale differences satisfy the conditional Lindeberg condition. By assumption (iv) in the theorem,
    \[
    \sup_{u\in[0,1]}\norm{\frac1T\sum_{t=0}^{\lfloor Tu\rfloor}\EB[\beps_t\beps_t^\top\mid\gF_{t-1}]-u\bGamma}\to0\qquad\text{almost surely}.
    \] 
    The martingale invariance principle~\citep{hall1980martingale} therefore yields $I_4\Rightarrow \bM^{-1}\bGamma^{1/2}\bB(\cdot)$, completing the proof. 
\end{proof}

\begin{proof}[Proof of Lemma~\ref{lem:consistency}]
    Define $\Lambda=\brc{\btheta^\blambda_m:\bm\eta_{\btheta^\blambda_m}=\bPi^\blambda_m\gT^\pi\bm\eta_{\btheta^\blambda_m},\blambda\in[0,1]^{\gS\times[m]}}$. By Theorem~8 and its asynchronous extension in Appendix~C of~\citet{rowland2023analysis}, with probability $1$, 
    \begin{equation*}
        \lim_{t\to\infty}\inf_{\btheta^*\in\Lambda}\norm{\btheta^{(t)}-\btheta^*}=0. 
    \end{equation*}
    Equation~\eqref{eq:common_quantile_fixed_point} implies that $\Lambda=\{\btheta_m\}$ and the conclusion follows. 
\end{proof}

\subsection{Proofs of Lemmas~\ref{lem:consistency_markov} and~\ref{lem:aux_FCLT}}\label{app:markov_fclt}\label{app:markov_consistency}\label{app:aux_fclt}
\begin{proof}[Proof of Lemma~\ref{lem:consistency_markov}]
We have
\begin{equation*}
    \btheta^{(t+1)}
    =\btheta^{(t)}-\alpha_t\bH^{(t+1)}\prn{\btheta^{(t)}}.
\end{equation*}
By construction,
\begin{equation*}
    \EB\brk{\bH^{(t+1)}(\btheta)\mid\gF_{t-1}}
    =\bD_{s^{(t+1)}}\bh(\btheta).
\end{equation*}
Define
\begin{align*}
    \bXi^{(t)}&\coloneqq\bH^{(t+1)}\prn{\btheta^{(t)}}-\bD_{s^{(t+1)}}\bh\prn{\btheta^{(t)}},\\
    \bV^{(t)}&\coloneqq\bXi^{(t)}+\bq\prn{\btheta^{(t)},s^{(t+1)}}.
\end{align*}
Lemma~\ref{lem:perturbation_tail} gives almost sure convergence of $\sum_{t\ge0}\alpha_t\bV^{(t)}$.

Define the stochastic-approximation clock
\begin{equation*}
    \varsigma_0=0,\qquad\varsigma_n=\sum_{t=0}^{n-1}\alpha_t,
\end{equation*}
and let $\what{\btheta}:[0,\infty)\to\RB^{\gS\times[m]}$ be the linear interpolation satisfying $\what{\btheta}(\varsigma_n)=\btheta^{(n)}$. 
By Lemma~\ref{lem:perturbed_solution_consistency}, this interpolation is a bounded perturbed solution of
\begin{equation}\label{eq:limiting_DI_consistency}
 \dot{\btheta}=-\bD_{d^\pi}\bh(\btheta).
\end{equation}

Theorems 4.2 and 4.3 in~\citet{benaim2006dynamics} now imply that the limit set of $\what{\btheta}$ is almost surely internally chain transitive for~\eqref{eq:limiting_DI_consistency}.

It remains to identify the internally chain transitive sets of Equation~\eqref{eq:limiting_DI_consistency}. By~\eqref{eq:common_quantile_fixed_point}, the QDP fixed-point family considered by~\citet{rowland2023analysis} reduces in our setting to the singleton, namely
\begin{equation*}
    \Lambda\coloneqq\brc{\btheta^\blambda_m:\bm\eta_{\btheta^\blambda_m}=\bPi^\blambda_m\gT^\pi\bm\eta_{\btheta^\blambda_m},\blambda\in[0,1]^{\gS\times[m]}}=\brc{\btheta_m}.
\end{equation*}
Accordingly, the Lyapunov function in~\citet[Proposition~18]{rowland2023analysis} becomes $L(\btheta)=\norm{\btheta-\btheta_m}_\infty$. Appendix C of~\citet{rowland2023analysis} shows that $L$ remains a Lyapunov function of the differential equation~\eqref{eq:limiting_DI_consistency} since $\min_sd^\pi(s)>0$. 

Finally, $L(\{\btheta_m\})=\{0\}$ has empty interior in $\RB$. \citet[Proposition~3.27]{benaim2006dynamics} therefore implies that every internally chain transitive set of Equation~\eqref{eq:limiting_DI_consistency} is contained in $\{\btheta_m\}$. Thus the limit set of $\what{\btheta}$ is almost surely $\{\btheta_m\}$, which completes the proof.
\end{proof}

\begin{proof}[Proof of Lemma~\ref{lem:aux_FCLT}]
    Use the localization events $G_T$ from Lemma~\ref{lem:second_order_remainder_markov}. According to the proof of Theorem~\ref{thm:general_FCLT}, we only need to verify that
    \begin{align}
        &\EB\brk{\norm{\wtilde{\brho}_t}\ind_{G_T}}\leq C_T\alpha_t,&t\geq T;\label{eq:higher_order_remainder}\\
        &\EB\sup_{u\in[0,1]}\norm{T^{-\frac12}\sum_{t=0}^{\down{Tu}}(\bphi_t+\bpsi_t)}\to0,&T\to\infty;\label{eq:martingale_remainder}\\
        &T^{-\frac12}\sum_{t=0}^{\down{Tu}}\bzeta_t\Rightarrow\bD_{d^\pi}^{\frac12}\bGamma^{\frac12}\bB(\cdot),&T\to\infty.\label{eq:martingale_leading}
    \end{align}
    For Equation~\eqref{eq:higher_order_remainder}, we notice that
    \begin{equation*}
        \norm{\wtilde{\brho}_t-\brho_t}\lesssim \alpha_t+\norm{\btheta^{(t+1)}-\btheta^{(t)}}+(t+1)^{-1}\lesssim\alpha_t, 
    \end{equation*}
    where the first inequality follows from Lemma~\ref{lem:poisson_properties}. This estimate and Lemma~\ref{lem:second_order_remainder_markov} imply~\eqref{eq:higher_order_remainder}.

    To prove Equation~\eqref{eq:martingale_remainder}, we notice that
    \begin{align*}
        &\EB\brk{\norm{\bphi_t}^2\mid\gF_{t-1}}\\
        \lesssim&\EB\brk{\norm{\bv(\btheta^{(t)},s^{(t+2)})-\bv(\btheta_m,s^{(t+2)})}^2\mid\gF_{t-1}}\\
        &+\norm{\gP^\pi\bv(\btheta^{(t)},s^{(t+1)})-\gP^\pi\bv(\btheta_m,s^{(t+1)})}^2\\
        \lesssim&\norm{\btheta^{(t)}-\btheta_m}^2, 
    \end{align*}
    where the last inequality follows from Lemma~\ref{lem:poisson_properties}. Moreover, 
    \begin{align*}
        &\EB\brk{\norm{\bpsi_t}^2\mid\gF_{t-1}}\\
        \lesssim&\EB\brk{\norm{\bH^{(t+1)}(\btheta^{(t)})-\bH^{(t+1)}(\btheta_m)}^2\mid\gF_{t-1}}+\EB\brk{\norm{\bD_{s^{(t+1)}}\brk{\bh(\btheta^{(t)})-\bh(\btheta_m)}}^2\mid\gF_{t-1}}\\
        \lesssim&\norm{\btheta^{(t)}-\btheta_m}, 
    \end{align*}
    where the last inequality uses the property of indicator function, as well as the boundedness and Lipschitz continuity of $\bh(\btheta)$. Therefore, both $\EB[\|\bphi_t\|^2\mid\gF_{t-1}]$ and $\EB[\|\bpsi_t\|^2\mid\gF_{t-1}]$ converge to $0$ almost surely. By Doob's maximal inequality, we deduce that
    \begin{align*}
        &\EB\brk{\sup_{u\in[0,1]}\norm{T^{-\frac12}\sum_{t=0}^{\down{Tu}}(\bphi_t+\bpsi_t)}^2}\\
        \lesssim&T^{-1}\EB\norm{\sum_{t=0}^{T}(\bphi_t+\bpsi_t)}^2\\
        =&T^{-1}\sum_{t=0}^{T}\EB\norm{\bphi_t+\bpsi_t}^2\to0, 
    \end{align*}
    where the last step follows from bounded convergence theorem. Therefore Equation~\eqref{eq:martingale_remainder} holds. 

    Finally, a direct calculation shows that
    \begin{equation*}
        \EB\brk{\bzeta_t\bzeta_t^\top\mid\gF_{t-1}}=\bD_{s^{(t+1)}}\bGamma. 
    \end{equation*}
    By the ergodic theorem, 
    \begin{equation*}
        \frac{1}{T}\sum_{t=0}^{\down{Tu}}\EB\brk{\bzeta_t\bzeta_t^\top\mid\gF_{t-1}}\to u\bD_{d^\pi}\bGamma
    \end{equation*}
    uniformly almost surely. Equation~\eqref{eq:martingale_leading} follows from the multidimensional martingale invariance principle~\citep{hall1980martingale}.

    To complete the argument, fix a sufficiently large localization index $N$. Equation~\eqref{eq:higher_order_remainder} implies
    \[
      \EB\sum_{t=N}^{\infty}(t+1)^{-1/2}\norm{\widetilde{\brho}_t}\ind_{G_N}<\infty.
    \]
    Since $\PB(\bigcup_NG_N)=1$, Tonelli's theorem and Kronecker's lemma give $T^{-1/2}\sum_{t=0}^T\norm{\widetilde{\brho}_t}\to0$ almost surely. The martingale noise $\beps_t=\bphi_t+\bpsi_t+\bzeta_t$ is bounded. Applying the five-term decomposition from the proof of Theorem~\ref{thm:general_FCLT} to the auxiliary recursion $\{\wtilde\bDelta^{(t)}\}$, this remainder estimate controls $I_1$, and the same stable-filter and maximal-inequality bounds control $I_0,I_2,I_3$. Equations~\eqref{eq:martingale_remainder}--\eqref{eq:martingale_leading} give the limit of $I_4$, proving the claim. 
\end{proof}

\subsection{Proofs of Lemmas~\ref{lem:av_residual}--\ref{lem:av_visit_boundary}}\label{app:av_direct_lemmas}
\phantomsection\label{app:proof_av_residual}
\begin{proof}[Proof of Lemma~\ref{lem:av_residual}]
Fix $t\ge1$ and $\btheta\in C_t$. For $(s,i)$ with $n=N_s(t)>0$, let $k=\lfloor\log_2n\rfloor\le J_t$. If $b_k\ge1$, then $|\widehat h_{t,s,i}(\btheta)|\le b_k$ follows from the boundedness of the score.
Otherwise, suppose that $\sigma\widehat h_{t,s,i}(\btheta)>b_k$ for some $\sigma\in\{-1,1\}$. Since $\lambda_k=b_k/2$, the inequality $\log(1+x)\ge x-x^2$ for $|x|\le1/2$ gives
\begin{align*}
\log K_{t,s,i,k}^{\sigma}(\btheta)&\ge \sigma\lambda_k\sum_{u=1}^t Z_{u,s,i}(\btheta)-\lambda_k^2\sum_{u=1}^t Z_{u,s,i}(\btheta)^2\\
&> n\lambda_k b_k-n\lambda_k^2=\frac{nb_k^2}{4}\ge\ell_k.
\end{align*}
Consequently,
\[
E_t(\btheta)
\ge\frac{w_k}{2d}K_{t,s,i,k}^{\sigma}(\btheta)
>\frac{w_ke^{\ell_k}}{2d}
=B_\alpha,
\]
contradicting $\btheta\in C_t$. This proves the claimed bound.
\end{proof}

\phantomsection\label{app:proof_max-coordinate}
\begin{proof}[Proof of Lemma~\ref{lem:max-coordinate}]
Fix $\btheta\in\mathcal C$ and put $\delta=\norm{\btheta-\btheta_m}_\infty$. Suppose that $\delta>\max_{s,i}\rho_{s,i}$. If $\theta(s,i)-\theta_m(s,i)=\delta$ at the coordinate $(s,i)$, then for every successor $v$ and index $j$,
\[
 \theta(s,i)-\gamma\theta(v,j)
 \ge\theta_m(s,i)-\gamma\theta_m(v,j)+(1-\gamma)\delta
 \ge\theta_m(s,i)-\gamma\theta_m(v,j)+(1-\gamma)\rho_{s,i}.
\]

Since the indicators in~\eqref{eq:av_score} are nondecreasing in the thresholds
$\theta(s,i)-\gamma\theta(v,j)$, the preceding inequality implies that
\[
g_{s,i}(\btheta;a,r,v)
\ge
g_{s,i}(\btheta_m+\rho_{s,i}\mathbf{1};a,r,v)
\]
for every transition $(a,r,v)$ observed from state $s$.
Averaging over these transitions gives
\[
\widehat h_{t,s,i}(\btheta)
\ge
\widehat h_{t,s,i}(\btheta_m+\rho_{s,i}\mathbf{1})
>a_{s,i},
\]
which is a contradiction. If a maximal-error coordinate has error $-\delta$, the reversed threshold comparison gives
$\widehat h_{t,s,i}(\btheta)\le\widehat h_{t,s,i}(\btheta_m-\rho_{s,i}\bm1)<-a_{s,i}$,
the same contradiction. The conclusion follows.
\end{proof}

\phantomsection\label{app:proof_av_visit_boundary}
\begin{proof}[Proof of Lemma~\ref{lem:av_visit_boundary}]
Put
\[
 D_{t,s,i}(\bv)=Z_{t,s,i}(\bv)-I_{t,s}h_{s,i}(\bv),
 \qquad D^\Sigma_{t,s,i}(\bv)=\sum_{u=1}^tD_{u,s,i}(\bv).
\]
Conditional on $\mathscr G_t$, the first variable is centered. Since $D_{t,s,i}$ is bounded, conditional Hoeffding's lemma gives, for every $\eta\in\RB$,
\[
 \EB\left[
  \exp\brc{\eta D_{t,s,i}(\bv)-\frac{\eta^2I_{t,s}}{8}}
  \middle|\mathscr G_t\right]\le1.
\]
Therefore, for each fixed $\eta\in\RB$, the process
\begin{equation*}
 H_t(\eta;\bv)=\exp\{\eta D^\Sigma_{t,s,i}(\bv)-\eta^2N_s(t)/8\}
\end{equation*}
is a nonnegative supermartingale with respect to the global filtration $\{\mathscr F_t\}_{t\geq0}$, with $H_0(\eta;\bv)=1$. 

Fix $\sigma\in\{-1,1\}$ and consider the crossing event
\[
\mathcal A_\sigma\coloneqq\left\{\exists t\ge1:N_s(t)\ge n_0,\quad\sigma\bigl(\widehat h_{t,s,i}(\bv)-h_{s,i}(\bv)\bigr)\ge x\right\}.
\]
Since
\[
D^\Sigma_{t,s,i}(\bv)=N_s(t)\bigl(\widehat h_{t,s,i}(\bv)-h_{s,i}(\bv)\bigr),
\]
taking $\eta=4\sigma x$ gives, at any time $t$ satisfying the conditions in $\mathcal A_\sigma$,
\[
\begin{aligned}
\log H_t(4\sigma x;\bv)&=4\sigma xD^\Sigma_{t,s,i}(\bv)-2x^2N_s(t)\\
&\ge 2x^2N_s(t)\ge 2n_0x^2.
\end{aligned}
\]
Consequently, Ville's inequality yields
\[
\PB(\mathcal A_\sigma)\le\PB\left(\sup_{t\ge0}H_t(4\sigma x;\bv)\ge e^{2n_0x^2}\right)\le e^{-2n_0x^2},
\]
as claimed.
\end{proof}

\subsection{Proof of Lemma~\ref{lem:dh_histogram}}\label{app:dh_lemmas}
\begin{proof}
For a fixed state pair \((s,v)\), epoch \(j\), and reward bin \(\ell\), put
\[
 Y_u\coloneqq I_{u,s}\ind\{S'_{u,s}=v,\ R_{u,s}\in B_{j,\ell}\},\qquad
 X_{s,v,j,\ell}\coloneqq\sum_{u\in\mathcal T_j}Y_u.
\]
For \(u\in\mathcal T_j\), the reward-density bound gives
\[
 \EB[e^{Y_u}\mid\mathscr G_u]=1+(e-1)\EB[Y_u\mid\mathscr G_u]\le\exp\{(e-1)I_{u,s}C_0\Delta_j\}.
\]
Therefore, $\EB[e^{X_{s,v,j,\ell}}]\le \exp\{(e-1)C_0\,2^j\Delta_j\}$ and Markov's inequality yields
\begin{equation*}
 \PB\{X_{s,v,j,\ell}>(e-1)C_0\,2^j\Delta_j+x\}\le e^{-x},\qquad x>0.
\end{equation*}

Set $x_j=(\zeta\log 2+3)(j+1)$. Taking $x=x_j$ and applying a union bound over state pairs and reward bins gives
\begin{align*}
&\PB\left\{
\exists (s,v)\in\gS^2,\ 0\le\ell<L_j:
X_{s,v,j,\ell}>(e-1)C_0\,2^j\Delta_j+x_j
\right\}\\
\le&2|\gS|^2 2^{\zeta j}e^{-x_j}\le 2|\gS|^2 e^{-3(j+1)}.
\end{align*}
The bound is summable in $j$. By Borel--Cantelli lemma, there is an almost surely finite index
$j_0$ such that, for every $j\ge j_0$,
\[
\max_{s,v,\ell}X_{s,v,j,\ell} \le (e-1)C_0\,2^{(1-\zeta)j}+(\zeta\log 2+3)(j+1).
\]

For each threshold $z=\theta(s,i)-\gamma\theta(v,q)$, the indicators differ only for the bin satisfying $b^-<z\le b^+$, of which there is at most one per epoch. Summing over successors and averaging over the $m$ quantile indices gives
\[
A_t\le|\gS|\sum_{j\le J_t}\max_{s,v,\ell}X_{s,v,j,\ell}.
\]
Consequently, for $J_t\ge j_0$,
\[
\begin{aligned}
A_t\le& |\gS|\sum_{j=0}^{J_t}\max_{s,v,\ell}X_{s,v,j,\ell}\\
\le& |\gS|(e-1)C_0\sum_{j=0}^{J_t}2^{(1-\zeta)j}+|\gS|(\zeta\log 2+3)\sum_{j=0}^{J_t}(j+1)+K,
\end{aligned}
\]
where
\[
K\coloneqq |\gS|\sum_{j=0}^{j_0-1}\max_{s,v,\ell}X_{s,v,j,\ell}<\infty\qquad\text{almost surely}.
\]
Since $2^{J_t}\leq t$ and $\zeta<1$, we have
\[
\begin{aligned}
A_t&\le \frac{(e-1)C_0|\gS|}{1-2^{-(1-\zeta)}}t^{1-\zeta}+\frac{|\gS|(\zeta\log 2+3)}{2}(J_t+1)(J_t+2)+K\\
&=\frac{(e-1)C_0|\gS|}{1-2^{-(1-\zeta)}}t^{1-\zeta}+o(t^{1-\zeta}).
\end{aligned}
\]
Thus $A_t\le C_A t^{1-\zeta}$ eventually, with
\begin{equation*}
C_A\coloneqq 1+\frac{(e-1)C_0|\gS|}{1-2^{-(1-\zeta)}}.
\end{equation*}

For the residual bound, let $n=N_s(t)>0$ and $k=\lfloor\log_2n\rfloor\le J_t$. If $b(n)\ge1$, the result follows from $|g|\le1$. Otherwise $\lambda_k=b(n)/2$. If
$\widehat h_{t,s,i}(\btheta)> b(n)+A_{t,s}/n$, then the lower scores satisfy
$\sum_{u:I_{u,s}=1}g^-_{u,s,i}(\btheta)> nb(n)$. The calculation in Appendix~\ref{app:proof_av_residual} gives
\[
 \log\underline K^+_{t,s,i,k}(\btheta)> n\lambda_k b(n)-n\lambda_k^2=\frac{nb(n)^2}{4}\ge\ell_k.
\]
Therefore,
\[
\underline E_t(\btheta)\ge \frac{w_k}{2d}\underline K^+_{t,s,i,k}(\btheta)> \frac{w_ke^{\ell_k}}{2d}=B_\alpha,
\]
contradicting $\btheta\in\widetilde C_t$. For a negative empirical residual, apply the same calculation to $-g^+$. This proves~\eqref{eq:dh_residual}.
\end{proof}

\subsection{Proofs of Lemmas~\ref{lem:lil_transfer} and~\ref{lem:lil_qtd_submodel}}\label{app:lil_lemma_proofs}
\phantomsection\label{app:lil_transfer}
\begin{proof}[Proof of Lemma~\ref{lem:lil_transfer}]
We can suppose that $(I_t)_{t\geq1}$ are closed intervals without loss of generality. Since $\psi$ is strictly monotone on $J$, the set
\[
 D_t=\{u\in J:\psi(u)\in I_t\}
\]
is either empty or a closed subinterval of $J$. Define $T_t$ to be its midpoint when nonempty and a fixed point of $J$ otherwise. Since $T_t$ is a measurable function of the observations through time $t$ and the randomization used by the interval rule, Equation~\eqref{eq:lil_simulator} gives $T_t=\widetilde T_t(X_1,\ldots,X_t,U')$, where $U'$ combines the randomness used by the simulation and the interval rule. This combined randomness is independent of the Gaussian observations and has a distribution that does not depend on $\beta$. Therefore, by Lemma~\ref{lem:lil_gaussian}, outside a Lebesgue-null set of $\beta$, $\limsup_{t\to\infty}|T_t-\beta|/\rho_t\ge1$ almost surely. Fix such a $\beta$ and work on the event $\{\psi(\beta)\in I_t\text{ for all }t\}$. Then $\beta\in D_t$ and
\begin{equation*}
 |T_t-\beta|\le\frac12\operatorname{length}(D_t).
\end{equation*}
If $\limsup_{t\to\infty}\operatorname{length}(I_t)/\rho_t=\infty$, there is nothing to show. Otherwise $\operatorname{length}(I_t)\to0$. Writing $a_J=\min_{u\in J}|\psi'(u)|>0$, for any $u,v\in D_t$, we have
\[
a_J|u-v|\le|\psi(u)-\psi(v)|\le\operatorname{length}(I_t).
\]
Taking the supremum over $u,v\in D_t$ yields
\[
\operatorname{length}(D_t)\le\frac{\operatorname{length}(I_t)}{a_J}\to0.
\]

Since $\beta\in D_t$,
\[
\sup_{u\in D_t}|u-\beta|\le \operatorname{length}(D_t)\to0.
\]
For each $\epsilon\in(0,|\psi'(\beta)|)$, continuity of $\psi'$ implies eventually
\[
 \operatorname{length}(I_t)\ge (|\psi'(\beta)|-\epsilon)\operatorname{length}(D_t)\ge2(|\psi'(\beta)|-\epsilon)|T_t-\beta|.
\]
Therefore,
\[
\limsup_{t\to\infty}\frac{\operatorname{length}(I_t)}{\rho_t}\geq2(|\psi'(\beta)|-\epsilon)\limsup_{t\to\infty}\frac{|T_t-\beta|}{\rho_t}\geq2(|\psi'(\beta)|-\epsilon). 
\]
Letting $\epsilon\downarrow0$ completes the proof.
\end{proof}

\phantomsection\label{app:lil_submodel}
\begin{proof}[Proof of Lemma~\ref{lem:lil_qtd_submodel}]
By Lemma~\ref{lem:lil_perturbation}, there exist a state $s_0$, a polynomial density $p_0$ positive on $(0,1)$, a bounded polynomial $z$ with $\int_0^1 z(r)p_0(r)\,\rd r=0$, and an interval $J=[-b,b]$ with $b>0$, such that the family $\{\mathcal M_\beta:\beta\in J\}$ has reward density
\[
p_\beta(r)=p_0(r)\bigl[1+(\Phi(\beta)-1/2)z(r)\bigr]
\]
at state $s_0$ for every action, and density $p_0$ at all other states. Its transition kernel is uniform and independent of $\beta$. The family satisfies the required regularity assumptions, and the target $\psi(\beta)$ is continuously differentiable with $\min_{\beta\in J}|\psi'(\beta)|>0$.
It remains to construct the exact sequential simulation.

The two fixed densities
\[
 p_+(r)=p_0(r)\brk{1+\frac{z(r)}{2}},\qquad
 p_-(r)=p_0(r)\brk{1-\frac{z(r)}{2}}
\]
satisfy the exact mixture identity $p_\beta=\Phi(\beta)p_++[1-\Phi(\beta)]p_-$. Let $(X_t)_{t\ge1}$ be i.i.d. $\normal(\beta,1)$, and independently generate a collection of i.i.d. $\operatorname{Unif}(0,1)$ random variables for the simulation. Whenever a reward at state $s_0$ is required, draw it from $p_+$ if $X_t\ge0$ and from $p_-$ otherwise, using the corresponding fixed inverse CDF and an auxiliary uniform. Since $\PB_\beta(X_t\ge0)=\Phi(\beta)$, such a reward has a distribution with density $p_\beta$. Rewards at other states are drawn from $p_0$ using independent uniforms. All remaining random choices are generated using the auxiliary uniforms according to their prescribed, parameter-independent conditional distributions.

Under all three sampling mechanisms, each round requires at most one reward from $s_0$, so a single Gaussian observation $X_t$ suffices to simulate the round. Collecting the auxiliary uniforms and any procedure randomization into $U$ yields the parameter-independent simulation~\eqref{eq:lil_simulator}.
\end{proof}

%% file: Appendix_technical_lemma.tex
\subsection{Technical Lemma for Appendix~\ref{app:independent_fclt}}\label{app:sa_remainder}\label{section:proof_lem_second_oder_remainder}

\begin{lemma}\label{lem:second_order_remainder}
Under the assumptions of Theorem~\ref{thm:general_FCLT}, let
\(\brho_t\coloneqq\bh(\bZ_t)-\bM(\bZ_t-\bz^*)\). Then, almost surely,
\[
 \sum_{t=0}^\infty(t+1)^{-1/2}\norm{\brho_t}<\infty.
\]
\end{lemma}

\begin{proof}
We construct a Lyapunov function first. Since $\bM$ is uniformly repulsive, there exists a positive definite matrix $\bQ$ satisfying $\bM^\top\bQ+\bQ\bM=\bI$. In fact, the solution is given by
\begin{equation*}
    \bQ=\int_0^\infty e^{-\bM^\top t}e^{-\bM t}\rd t. 
\end{equation*}

Let $0<p_-\le p_+<\infty$ be the smallest and largest eigenvalues of $\bQ$. Then
\begin{equation*}
    p_-\norm{\bz}^2\leq V(\bz)\coloneqq\bz^{\top}\bQ\bz\leq p_+\norm{\bz}^2.
\end{equation*}

Choose $\kappa$ such that $0<\kappa\leq\min\{\delta,(4p_+L)^{-1}\}$. Then, for $\norm{\bz-\bz^*}\leq\kappa$, we have
\begin{align*}
&2(\bz-\bz^*)^{\top}\bQ \bh(\bz)
\notag\\
=&2(\bz-\bz^*)^{\top}\bQ\bM(\bz-\bz^*)+2(\bz-\bz^*)^{\top}\bQ\brk{\bh(\bz)-\bM(\bz-\bz^*)}\\
=&(\bz-\bz^*)^{\top}\prn{\bQ\bM+\bM^{\top}\bQ}(\bz-\bz^*)+2(\bz-\bz^*)^{\top}\bQ\brk{\bh(\bz)-\bM(\bz-\bz^*)}\\
=&\norm{\bz-\bz^*}^2+2(\bz-\bz^*)^{\top}\bQ\brk{\bh(\bz)-\bM(\bz-\bz^*)}\\
\geq&\norm{\bz-\bz^*}^2-2p_+L\norm{\bz-\bz^*}^3\\
\geq&\frac{1}{2}\norm{\bz-\bz^*}^2,
\end{align*}
Also, for $\norm{\bz-\bz^*}\leq\kappa$,
\begin{align*}
\norm{\bh(\bz)}&\leq\norm{\bM(\bz-\bz^*)}+\norm{\bh(\bz)-\bM(\bz-\bz^*)}\\
&\leq\left(\norm{\bM}+L\delta\right)\norm{\bz-\bz^*}.
\end{align*}
Denote $H=\norm{\bM}+L\delta$ and
\begin{align*}
&V\left(\bz-\bz^*-\varepsilon\bh(\bz)\right)-V\left(\bz-\bz^*\right)\\
=&-2\varepsilon(\bz-\bz^*)^\top\bQ\bh(\bz)+\varepsilon^2\bh(\bz)^\top\bQ\bh(\bz)\\
\leq&-\frac{\varepsilon}{2}\norm{\bz-\bz^*}^2+\varepsilon^2 p_+ H^2\norm{\bz-\bz^*}^2.
\end{align*}

Since $\alpha_t\to0$, there exists $t_0$ such that $\alpha_tp_+H^2\le\frac14$ and $c\alpha_t\leq 1$ for all $t\geq t_0$, where $c\coloneqq(4p_+)^{-1}>0$. Thus, for every $t\ge t_0$ and every $\norm{\bz-\bz^*}\leq\kappa$,
\begin{align*}
V(\bz-\bz^*-\alpha_t\bh(\bz))\leq&V(\bz-\bz^*)-\frac{\alpha_t}{4}\norm{\bz-\bz^*}^2\\
\leq&(1-c\alpha_t)V(\bz-\bz^*).
\end{align*}

Now we prove the desired conclusion. Fix integer $T\ge t_0$ and define the stopping time
\begin{equation*}
\tau_{T}\coloneqq\inf\brc{t\geq T:\norm{\bZ_t-\bz^*}>\kappa},
\end{equation*}
with $\inf\varnothing=\infty$, and the event $G_{T}\coloneqq\{\tau_{T}=\infty\}$. 

For $t\geq T$, the event $\{\tau_T>t\}$ is $\gF_{t-1}$-measurable, and
\begin{align*}
&\EB\brk{V(\bZ_{t+1}-\bz^*)\ind\{\tau_T>t+1\}\mid\gF_{t-1}}\\
\leq&\ind\{\tau_T>t\}\EB\brk{V(\bZ_{t+1}-\bz^*)\mid\gF_{t-1}}\\
=&\ind\{\tau_T>t\}\Bigl[V\prn{\bZ_t-\bz^*-\alpha_t\bh(\bZ_t)}
 +\alpha_t^2\EB\brk{\beps_t^\top\bQ\beps_t\mid\gF_{t-1}}\Bigr]\\
\leq&(1-c\alpha_t)V(\bZ_t-\bz^*)\ind\{\tau_T>t\}+C\alpha_t^2.
\end{align*}
Define $u_t=\EB\brk{V(\bZ_t-\bz^*)\ind\{\tau_T>t\}}$ for $t\geq T$. Then
\begin{equation*}
u_{t+1}\leq(1-c\alpha_t)u_t+C\alpha_t^2.
\end{equation*}
By~\citet[Lemma~A.12]{li2021statistical}, for $t\geq T$, $u_t\leq C_T\alpha_t$. 

Since $G_T\subseteq\{\tau_T> t\}$, we know that for $t\geq T$, 
\begin{equation*}
    \EB\brk{\|\brho_t\|\ind_{G_T}}\lesssim\EB\brk{\|\bZ_t-\bz^*\|^2\ind_{G_T}}\leq C^\prime_T\alpha_t. 
\end{equation*}
Therefore, we know that
\begin{equation*}
    \EB\sum_{t=T}^{\infty}\frac{\norm{\brho_t}\ind_{G_T}}{\sqrt{t+1}}\lesssim\sum_{t=T}^{\infty}(t+1)^{-(a+\frac12)}<\infty. 
\end{equation*}
Hence on $G_T$, we know that
\begin{equation*}
    \sum_{t=T}^{\infty}\frac{\norm{\brho_t}}{\sqrt{t+1}}<\infty
\end{equation*}
and the conclusion holds. However, since $\bZ_t\to\bz^*$ almost surely, 
\begin{equation*}
    \PB\prn{\bigcup_{T=1}^\infty G_T}=1. 
\end{equation*}
Therefore, we know that with probability $1$, 
\begin{equation*}
    \sum_{t=0}^{\infty}\frac{\norm{\brho_t}}{\sqrt{t+1}}<\infty.\qedhere
\end{equation*}
\end{proof}

\subsection{Technical Lemmas for Appendix~\ref{app:markov_fclt}}\label{app:sa_technical}

\begin{lemma}\label{lem:poisson_properties}
For the Poisson solution in Proposition~\ref{prop:poisson_eq}, there exists $C<\infty$ such that
\[
 \max_s\norm{\bv(\btheta,s)-\bv(\widetilde{\btheta},s)}
 \le C\norm{\btheta-\widetilde{\btheta}}
\]
for all $\btheta,\widetilde{\btheta}$. Moreover, $\bv(\btheta_m,s)=0$ for every $s$.
\end{lemma}

\begin{proof}[Proof of Proposition~\ref{prop:poisson_eq} and Lemma~\ref{lem:poisson_properties}]
Denote $\bq(\btheta,s)=(\bD_s-\bD_{d^\pi})\bh(\btheta)$. On the finite-dimensional subspace
\[
\gC_0=\brc{\bm f\in\prn{\RB^{\gS\times[m]}}^\gS:\sum_s d^\pi(s)\bm f(s)=0},
\]
the operator $\gI-\gP^\pi$ is invertible.  Indeed, if $(\gI-\gP^\pi)\bm f=0$, then each coordinate of $\bm f$ is harmonic for an irreducible finite-state chain and is therefore constant, and the centering condition forces this constant to be zero. Since $\gC_0$ is finite dimensional, the inverse, which is denoted by $(\gI-\gP^\pi)^{-1}_{\gC_0}$, is bounded. Thus
\[
\bv(\btheta,\cdot)=(\gI-\gP^\pi)^{-1}_{\gC_0}\bq(\btheta,\cdot).
\]
The Lipschitz continuity and boundedness of $\bq$ imply the same properties for $\bv$. Finally, $\bh(\btheta_m)=0$ implies $\bq(\btheta_m,\cdot)=0$, therefore $\bv(\btheta_m,\cdot)=0$.
\end{proof}

\begin{lemma}\label{lem:boundedness}
    For the Markovian QTD recursion~\eqref{eq:markov_QTD}, define
    \begin{equation*}
        B\coloneqq \max\brc{\frac{\alpha_0}{1-\gamma},-\min_{s,i}\theta^{(0)}(s,i),\max_{s,i}\theta^{(0)}(s,i)-\frac{1}{1-\gamma}}. 
    \end{equation*}
    Then
    \begin{equation*}
        \btheta^{(t)}\in\brk{-B,\frac{1}{1-\gamma}+B}^{\gS\times[m]}. 
    \end{equation*}
\end{lemma}

\begin{proof}
Set $V_{\max}=(1-\gamma)^{-1}$. By the definition of $B$,
\begin{equation*}
    \btheta^{(0)}\in[-B,V_{\max}+B]^{\gS\times[m]},
    \qquad
    (1-\gamma)B\geq\alpha_0\geq\alpha_t.
\end{equation*}
Suppose inductively that $\btheta^{(t-1)}$ belongs to this box and consider an updated coordinate $\prn{s^{(t)},i}$. Since the reward is supported on $[0,1]$,
\begin{equation*}
    -\gamma B\leq r^{(t)}+\gamma\theta^{(t-1)}\prn{s^{(t+1)},j}\leq1+\gamma\prn{V_{\max}+B}=V_{\max}+\gamma B.
\end{equation*}
Furthermore,
\begin{equation*}
    -\gamma B\geq-B+\alpha_{t-1},\qquad V_{\max}+\gamma B\leq V_{\max}+B-\alpha_{t-1}.
\end{equation*}
If $\theta^{(t-1)}(s^{(t)},i)<-B+\alpha_{t-1}$, then all indicators in the QTD update equal $0$, therefore, 
\begin{equation*}
    \theta^{(t-1)}(s^{(t)},i)<\theta^{(t)}(s^{(t)},i)<-B+2\alpha_{t-1}\leq V_{\max}+B. 
\end{equation*}
Similarly, if $\theta^{(t-1)}(s^{(t)},i)>V_{\max}+B-\alpha_{t-1}$, then all indicators equal $1$ and 
\begin{equation*}
    -B\leq V_{\max}+B-2\alpha_{t-1}\leq\theta^{(t)}(s^{(t)},i)<\theta^{(t-1)}(s^{(t)},i). 
\end{equation*}
Otherwise, 
\begin{equation*}
    -B\leq\theta^{(t-1)}(s^{(t)},i)-\alpha_{t-1}\leq\theta^{(t)}(s^{(t)},i)\leq\theta^{(t-1)}(s^{(t)},i)+\alpha_{t-1}\leq V_{\max}+B. 
\end{equation*}
The result follows by induction.
\end{proof}

\begin{lemma}\label{lem:perturbation_tail}
For Markovian QTD, let
\[
 \bV^{(t)}\coloneqq\bH^{(t+1)}(\btheta^{(t)})
                 -\bD_{d^\pi}\bh(\btheta^{(t)}).
\]
If \(\alpha_t=c(t+1)^{-a}\) with \(1/2<a<1\), then
\(\sum_{t=0}^\infty\alpha_t\bV^{(t)}\) converges almost surely.
\end{lemma}

\begin{proof}
Let $\bXi^{(t)}\coloneqq\bH^{(t+1)}(\btheta^{(t)})-\bD_{s^{(t+1)}}\bh(\btheta^{(t)})$. Since it is a bounded martingale difference sequence, the martingale
\begin{equation*}
    \sum_{t=0}^{n}\alpha_t\bXi^{(t)}
\end{equation*}
is $L^2$ bounded and thus converges almost surely.

By Poisson equation and Equation~\eqref{eq:markov_decomposition}, $\bq(\btheta^{(t)},s^{(t+1)})=\bw_t+\bphi_t$. By Lemma~\ref{lem:poisson_properties}, $\bphi_t$ is also a bounded martingale difference sequence and $\sum_{t=0}^{n}\alpha_t\bphi_t$ converges almost surely.

For $\bw_t$, we denote $\bv_t=\bv(\btheta^{(t)},s^{(t+1)})$. Summation by parts gives that, for $n\leq\ell$,
\begin{align*}
    &\sup_{\ell>n}\norm{\sum_{t=n}^{\ell}\alpha_t(\bv_t-\bv_{t+1})}\\
    =&\sup_{\ell>n}\norm{\alpha_n\bv_n-\alpha_\ell\bv_{\ell+1}
    +\sum_{t=n+1}^{\ell}(\alpha_t-\alpha_{t-1})\bv_t}\\
    \lesssim&2\alpha_n+\sum_{t=n+1}^\infty (t+1)^{-(1+a)}\to0.
\end{align*}
Finally, by Lipschitz continuity of $\bv$ in $\btheta$,
\begin{align*}
    &\sup_{\ell>n}\norm{\sum_{t=n}^{\ell}\alpha_t(\bw_t-\bv_t+\bv_{t+1})}\\
    \lesssim&\sup_{\ell>n}\sum_{t=n}^{\ell}\alpha_t\norm{\btheta^{(t+1)}-\btheta^{(t)}}\\
    \lesssim&\sum_{t=n}^\infty\alpha_t^2\to0.
\end{align*}
Therefore, $\sum_{t=0}^\infty\alpha_t\bq(\btheta^{(t)},s^{(t+1)})$ converges almost surely and the conclusion follows. 
\end{proof}

\begin{lemma}\label{lem:perturbed_solution_consistency}
Define the stochastic-approximation clock $\varsigma_0=0$ and $\varsigma_n=\sum_{t=0}^{n-1}\alpha_t$. Let $\what{\btheta}$ be the piecewise-linear interpolation satisfying $\what{\btheta}(\varsigma_n)=\btheta^{(n)}$ for every $n\ge0$. Almost surely, \(\widehat{\btheta}\) is a bounded perturbed solution of~\eqref{eq:limiting_DI_consistency} in the sense of~\citet{benaim2006dynamics}.
\end{lemma}

\begin{proof}
First, we know that
\begin{equation*}
    \btheta^{(t+1)}
    =\btheta^{(t)}-\alpha_t\brk{\bD_{d^\pi}\bh\prn{\btheta^{(t)}}+\bV^{(t)}}.
\end{equation*}
Lemma~\ref{lem:boundedness} implies that $\what{\btheta}$ is bounded. For $u\in[\varsigma_t,\varsigma_{t+1})$, set
\begin{equation*}
    \bar{\btheta}(u)=\btheta^{(t)},
    \qquad
    \bU(u)=-\bV^{(t)}.
\end{equation*}
Then $\what{\btheta}$ is absolutely continuous and, for almost every such $u$,
\begin{equation*}
    \dot{\what{\btheta}}(u)-\bU(u)
    =-\bD_{d^\pi}\bh\prn{\bar{\btheta}(u)}.
\end{equation*}
The one-step QTD update is uniformly bounded, and therefore
\begin{equation}\label{eq:interpolation_left_endpoint}
    \sup_{u\in[\varsigma_t,\varsigma_{t+1}]}
    \norm{\what{\btheta}(u)-\bar{\btheta}(u)}
    \leq C\alpha_t\longrightarrow0.
\end{equation}
Denote $\gH(\btheta)=\{-\bD_{d^\pi}\bh(\btheta)\}$ and recall the $\delta$-enlargement used by~\citet{benaim2006dynamics},
\begin{equation*}
    \gH^\delta(\btheta)
    =\brc{\bz:\ \exists\widetilde{\btheta},\
    \norm{\widetilde{\btheta}-\btheta}<\delta,\
    \operatorname{dist}\prn{\bz,\gH(\widetilde{\btheta})}<\delta}.
\end{equation*}
Taking $\widetilde{\btheta}=\bar{\btheta}(u)$ in this definition and using Equation~\eqref{eq:interpolation_left_endpoint}, there exists a deterministic piecewise-constant function $\delta(u)\downarrow0$ such that
\begin{equation*}
    \dot{\what{\btheta}}(u)-\bU(u)
    \in\gH^{\delta(u)}\prn{\what{\btheta}(u)}
\end{equation*}
for almost every sufficiently large $u$.

It remains to verify that as $u\to\infty$,
\begin{equation*}
    \sup_{0\leq v\leq T}
    \norm{\int_u^{u+v}\bU(r)\rd r}
    \to0
    \qquad\text{as }u\to\infty.
\end{equation*}
However, according to Lemma~\ref{lem:perturbation_tail}, almost surely we have
\begin{align*}
    &\lim_{u\to\infty}\sup_{0\leq v\leq T}
    \norm{\int_u^{u+v}\bU(r)\rd r}\\
    \leq&\lim_{n\to\infty}\prn{\sup_{l>n}\norm{\sum_{t=n}^l\alpha_t\bV^{(t)}}+2\alpha_n\sup_{t\geq1}\norm{\bV^{(t)}}}\to0
\end{align*}
Therefore, we have verified the conditions for perturbed solutions in~\citet{benaim2006dynamics} and the proof is completed.
\end{proof}

\begin{lemma}\label{lem:second_order_remainder_markov}
For Markovian QTD, fix a sufficiently small \(\delta>0\) and define
\[
 \tau_T\coloneqq\inf\{t\ge T:\norm{\btheta^{(t)}-\btheta_m}>\delta\},\qquad
 G_T\coloneqq\{\tau_T=\infty\},
\]
with \(\inf\varnothing=\infty\). For every sufficiently large \(T\), there is \(C_T<\infty\) such that
\[
 \EB\brk{\norm{\bDelta^{(t)}}^2\ind_{G_T}}\le C_T\alpha_t,
 \qquad t\ge T.
\]
\end{lemma}

\begin{proof}
We construct the same Lyapunov function as in the proof of Lemma~\ref{lem:second_order_remainder}. Define $V(\bz)=\bz^\top\bQ \bz$ where $\bM^\top\bQ+\bQ\bM=\bI$. Let $0<p_-\le p_+$ be the minimal and maximal eigenvalues of $\bQ$.
By Lemma~\ref{lem:poisson_properties}, for all sufficiently large $t$, 
\begin{equation*}
    \frac12\norm{\bDelta^{(t)}}\leq\norm{\widetilde{\bDelta}^{(t)}}\le2\norm{\bDelta^{(t)}}.
\end{equation*}

By Lemma~\ref{lem:poisson_properties}, on the event $\{\tau_T> t\}$, 
\begin{equation}\label{eq:rhotilde_x}
    \norm{\widetilde{\brho}_t}\le C\brk{\norm{\widetilde{\bDelta}^{(t)}}^2+\alpha_t}, 
\end{equation}
for sufficiently large $T$. 

Set $\bx=\widetilde{\bDelta}^{(t)}$. As in Section~\ref{section:proof_lem_second_oder_remainder}, for all large $t$,
\begin{equation*}
    V((\bI-\alpha_t\bM)\bx)\le V(\bx)-c_0\alpha_t\norm{\bx}^2.
\end{equation*}
Let $\bz=(\bI-\alpha_t\bM)\bx$.  Since $\beps_t$ is a bounded martingale difference sequence,
\begin{equation*}
    \EB\brk{V(\bz-\alpha_t\beps_t)\mid\gF_{t-1}}=V(\bz)+\alpha_t^2\EB\brk{\beps_t^\top\bQ\beps_t\mid\gF_{t-1}}\le V(\bz)+C\alpha_t^2.
\end{equation*}

We bound the contribution of $\widetilde{\brho}_t$ pathwise before taking conditional expectations:
\begin{align*}
    &\abs{V(\bz-\alpha_t\beps_t-\alpha_t\widetilde{\brho}_t)-V(\bz-\alpha_t\beps_t)}\\
    \le&2p_+\alpha_t\norm{\bz-\alpha_t\beps_t}\norm{\widetilde{\brho}_t}+p_+\alpha_t^2\norm{\widetilde{\brho}_t}^2.
\end{align*}
On $\{\tau_T>t\}$, use boundedness of $\beps_t$ and Equation~\eqref{eq:rhotilde_x}, for sufficiently large $T$ and $t>T$,
\begin{align*}
    &2p_+\alpha_t\norm{\bz-\alpha_t\beps_t}\norm{\widetilde{\brho}_t}
    +p_+\alpha_t^2\norm{\widetilde{\brho}_t}^2\\
    \leq& 2Cp_+\alpha_t(\delta+C\alpha_t)\norm{\bx}^2+C\alpha_t^2\\
    \leq& \frac{c_0}{2}\alpha_t\norm{\bx}^2+C\alpha_t^2,
\end{align*}
where the last inequality follows because $\delta$ is sufficiently small and $T$ is sufficiently large. 

Therefore, there exist $c_1,C>0$ such that
\begin{equation*}
    \EB\brk{V(\widetilde{\bDelta}^{(t+1)})\ind\{\tau_T>t+1\}\mid\gF_{t-1}}
    \le(1-c_1\alpha_t)V(\widetilde{\bDelta}^{(t)})\ind\{\tau_T>t\}+C\alpha_t^2.
\end{equation*}
Define
\[
u_t=\EB\brk{V(\widetilde{\bDelta}^{(t)})\ind\brc{\tau_T>t}}.
\]
Then
\begin{equation*}
    u_{t+1}\le(1-c_1\alpha_t)u_t+C\alpha_t^2.
\end{equation*}
\citet[Lemma~A.12]{li2021statistical} yields $u_t\le K\alpha_t$.  Since $G_T\subseteq\{\tau_T>t\}$, Lemma~\ref{lem:second_order_remainder_markov} holds.
\end{proof}

\subsection{Technical Lemmas for Appendix~\ref{app:lil_lemma_proofs}}\label{app:lil_technical}
\phantomsection\label{app:lil_gaussian}
\begin{lemma}
\label{lem:lil_gaussian}
Write $\rho_n\coloneqq\sqrt{2\log\log n/n}$ for $n\ge3$. Let $X_1,X_2,\ldots$ be i.i.d. $\normal(\beta,1)$ under $\PB_\beta$, and let $U$ be independent of the observations and the parameter, with a parameter-independent distribution. For any real-valued estimator sequence
$T_n=T_n(X_1,\ldots,X_n,U)$, for Lebesgue-almost every $\beta\in\RB$,
\begin{equation*}
 \PB_\beta\!\left(\limsup_{n\to\infty}
              \frac{|T_n-\beta|}{\rho_n}\ge1\right)=1.
\end{equation*}

\end{lemma}

\begin{proof}
Introduce a random parameter $B\sim\normal(0,1)$, independent errors $\epsilon_i\sim\normal(0,1)$, and the independent seed $U$. Set $X_i=B+\epsilon_i$ and denote their joint law by $\mathbb Q$. Then, conditional on $B=\beta$, the observations $(X_i)_{i\ge1}$ and the auxiliary randomization $U$ have joint law $\PB_\beta$. For $\mathscr H_n=\sigma(U,X_1,\ldots,X_n)$, Gaussian conditioning gives
\begin{equation*}
 M_n\coloneqq\EB_{\mathbb Q}[B\mid\mathscr H_n]
       =\frac{\sum_{i=1}^nX_i}{n+1},\qquad
 v_n\coloneqq\operatorname{Var}_{\mathbb Q}(B\mid\mathscr H_n)
       =\frac1{n+1}.
\end{equation*}
The process $M_n$ is a Gaussian martingale. If $N>n$, the increment $M_N-M_n$ is independent of $\mathscr H_n$ and has distribution $\normal(0,v_n-v_N)$. To verify independence, we notice that $M_N-M_n$ and $(X_1,\ldots,X_n)$ are jointly Gaussian, and their covariances vanish by the martingale property. Then the independence of $U$ gives the assertion. Also $B-M_n\sim\normal(0,v_n)$ under $\mathbb Q$.

Fix an integer $q\ge2$ and $a\in(0,1)$. Put
\[
 n_k=q^k,\quad
 D_k=M_{n_{k+1}}-M_{n_k},\quad
 s_k^2=v_{n_k}-v_{n_{k+1}},\quad
 A_k=T_{n_k}-M_{n_k}.
\]
The variable $A_k$ is $\mathscr H_{n_k}$-measurable. By symmetry and unimodality of the Gaussian density, define
\[
 E_k=\brc{|D_k-A_k|\ge\sqrt{2a s_k^2\log k}},\qquad k\ge2
\]
and we have
\begin{equation*}
 \mathbb Q(E_k\mid\mathscr H_{n_k})
 \ge p_k\coloneqq \PB\{|Z|\ge\sqrt{2a\log k}\},
 \qquad Z\sim\normal(0,1).
\end{equation*}
For all sufficiently large $k$, a Gaussian tail bound gives $p_k\ge c_a k^{-a}(\log k)^{-1/2}$. Since $a\in(0,1)$, we have
\[
\sum_{k=2}^\infty\QB(E_k\mid\mathscr H_{n_k})=\infty\qquad\QB\text{-almost surely}.
\]
As $E_k\in\mathscr H_{n_{k+1}}$, the conditional Borel--Cantelli lemma~\citep[Theorem~4.3.4]{durrett2019probability} implies that $E_k$ occurs infinitely often, $\QB$-almost surely.

On the other hand, the marginal Gaussian law of $B-M_{n_{k+1}}$ yields
\[
 \mathbb Q\left(|B-M_{n_{k+1}}|>\sqrt{4v_{n_{k+1}}\log k}\right)\le2k^{-2}.
\]
By the first Borel--Cantelli lemma, $|B-M_{n_{k+1}}|\leq\sqrt{4v_{n_{k+1}}\log k}$ eventually. Consequently, almost surely, for infinitely many $k$,
\begin{align}
 |B-T_{n_k}|
 &\ge |D_k-A_k|-|B-M_{n_{k+1}}| \notag\\
 &\ge \sqrt{2a(v_{n_k}-v_{n_{k+1}})\log k}
       -\sqrt{4v_{n_{k+1}}\log k}.\label{eq:lil_block_comparison}
\end{align}
Now $n_kv_{n_k}\to1$, $n_kv_{n_{k+1}}\to q^{-1}$, and
$\log\log n_k=\log k+\log\log q$. Dividing~\eqref{eq:lil_block_comparison} by $\rho_{n_k}$ gives
\[
 \limsup_{n\to\infty}\frac{|B-T_n|}{\rho_n}
 \ge\sqrt{a(1-q^{-1})}-\sqrt{2/q},
 \qquad\mathbb Q\text{-almost surely}.
\]
Sending $q\to\infty$ and then $a\to1$ gives that
\[
\limsup_{n\to\infty}\frac{|B-T_n|}{\rho_n}\ge1,\qquad\mathbb Q\text{-almost surely}.
\]

Finally, let $\varphi$ denote the standard normal density. Since the observations conditional on $B=\beta$ have law $\PB_\beta$, the preceding almost-sure result gives
\[
\int_{\RB}\PB_\beta\left(\limsup_{n\to\infty}\frac{|T_n-\beta|}{\rho_n}<1\right)\varphi(\beta)\rd\beta=0.
\]
Therefore the conclusion holds since $\varphi(\beta)>0$ for every $\beta\in\RB$.
\end{proof}

\phantomsection\label{app:lil_perturbation}

\begin{lemma}\label{lem:lil_perturbation}
Fix the structural quantities and nonzero contrast $\bm c$ in Theorem~\ref{thm:lil_lower}. There exist a state $s_0\in\gS$, a polynomial density $p_0$ on $[0,1]$ that is strictly positive on $(0,1)$, a polynomial $z$ satisfying
\[
|z(r)|\le1,\qquad\int_0^1 z(r)p_0(r)\,\rd r=0,
\]
and $b>0$, with the following properties. Let $J=[-b,b]$ and define a family of MDPs $\{\gM_\beta:\beta\in J\}$ with transition kernel
\[
P_\beta(s'\mid s,a)=\frac1{|\gS|}
\]
and reward densities
\[
p^\beta_{s,a}(r)=
\begin{cases}
p_0(r)[1+\{\Phi(\beta)-1/2\}z(r)], & s=s_0,\\
p_0(r), & s\ne s_0,
\end{cases}
\]
for every action $a$, where $\Phi$ denotes the standard normal CDF. Then we have
\begin{enumerate}
\item Every $\gM_\beta$ satisfies Assumptions~\ref{assump:density}--\ref{assump:recurrence}. The reward densities are uniformly bounded and Lipschitz on $[0,1]$, with constants independent of $\beta\in J$.
\item The target
$\psi(\beta)\coloneqq\bm c^\top\btheta_m(\gM_\beta)$ is continuously differentiable on a neighborhood of $J$, with
\[
\min_{\beta\in J}|\psi'(\beta)|>0.
\]
\end{enumerate}
\end{lemma}

\begin{proof}
\noindent\textbf{Step 1: A common, concentrated reward density.}
Choose
\[
\epsilon_\gamma=\frac{1-\gamma}{4(1+\gamma)},\qquad
a_\star=\left\lceil\frac{m}{4\epsilon_\gamma^2}\right\rceil,
\qquad
p_0(r)=\frac{r^{a_\star-1}(1-r)^{a_\star-1}}
{\mathrm B(a_\star,a_\star)},\quad 0\le r\le1.
\]
Here $\mathrm B(\cdot,\cdot)$ is the beta function. The density $p_0$ is positive on $(0,1)$, polynomial and Lipschitz on $[0,1]$, and vanishes at the two endpoints. Set $p_{s,a}=p_0$ for every $(s,a)$, and $P(s'\mid s,a)=1/|\gS|$ for every $(s,a,s^\prime)$.

For $R\sim p_0$, $\EB_0R=1/2$ and $\operatorname{Var}_0(R)=[4(2a_\star+1)]^{-1}$. Therefore Chebyshev's inequality gives
\[
\PB_0\prn{\abs{R-\frac12}\ge\epsilon_\gamma}\le\frac{1}{4(2a_\star+1)\epsilon_\gamma^2}<\frac{1}{2m}.
\]
In particular, $F_0^{-1}(\tau_1),F_0^{-1}(\tau_m)\in(1/2-\epsilon_\gamma,1/2+\epsilon_\gamma)$.

Let $\gM_0$ denote this base MDP and write $\btheta_0\coloneqq\btheta_m(\gM_0)$. By symmetry and uniqueness of the projected Bellman fixed point, its coordinates are the same at every state. We write $\theta_0(s,i)=q_i$. 
At the fixed point, $q_i$ is the $\tau_i$-quantile of $R+\gamma q_K$, where $R\sim p_0$ and $K$ is uniform on $[m]$, independently of $R$. Since $R+\gamma q_1\le R+\gamma q_K\le R+\gamma q_m$, monotonicity of quantiles gives
\[
F_0^{-1}(\tau_i)+\gamma q_1\le q_i\le F_0^{-1}(\tau_i)+\gamma q_m.
\]
Therefore, 
\[
q_1\ge\frac{F_0^{-1}(\tau_1)}{1-\gamma},
\qquad
q_m\le\frac{F_0^{-1}(\tau_m)}{1-\gamma}.
\]
Consequently, for all $i,j$, $t_{ij}\coloneqq q_i-\gamma q_j\in(1/4,3/4)$ and Assumption~\ref{Assumption_no_boundary_quantile} holds for the base model.

\par\medskip\noindent\textbf{Step 2: A nondegenerate smooth reward perturbation.}
Define
\[
Y_i(r)=\frac1m\sum_{j=1}^m\ind\{r<t_{ij}\},
\qquad
\EB_0Y_i(R)=\tau_i,
\qquad
d_i=\frac1m\sum_{j=1}^m p_0(t_{ij})>0.
\]
Let $\bG_0\coloneqq\nabla\bh_0(\btheta_0)$ be the Jacobian of the population Bellman score for the base model. By~\eqref{eq:bellman_jacobian_entries} and~\eqref{eq:bellman_row_sum}, $\bG_0$ is strictly diagonally dominant, and hence invertible. Set $\bv=\bG_0^{-\top}\bm c\ne\bm0$ and choose a state $s_0$ with $\bv_{s_0}\ne\bm0$. Define
\[
H(r)=\sum_{i=1}^m v_{s_0,i}\{Y_i(r)-\tau_i\}.
\]
Then $\EB_0H(R)=0$. Let $i_\star$ be the largest index such that $v_{s_0,i_\star}\ne0$. Since $q_1<\cdots<q_m$, the largest threshold appearing in $H$ is uniquely $t_{i_\star,1}=q_{i_\star}-\gamma q_1$, at which $H$ has a nonzero jump of size $-v_{s_0,i_\star}/m$. As $p_0$ is strictly positive on $(0,1)$, we have $\EB_0H(R)^2>0$.

Consequently, there exists an integer $k_0\ge1$ with $\EB_0[H(R)R^{k_0}]\ne0$. Indeed, if $\EB_0[H(R)R^k]=0$ for every $k$, we know that $\EB_0[H(R)f(R)]=0$ for every continuous $f$. Therefore $H(r)=0$ almost surely, contradicting $\EB_0H(R)^2>0$. Set $z(r)=r^{k_0}-\EB_0R^{k_0}$, then $\EB_0z(R)=0$, $|z(r)|\le1$ and $\EB_0[H(R)z(R)]\ne0$. 
Let $\Phi$ and $\varphi$ denote the standard normal CDF and density. For every action $a$, define
\begin{equation*}
p_{s_0,a}(r)=p_\beta(r)\coloneqq p_0(r)\left[1+\brc{\Phi(\beta)-\frac12}z(r)\right],
\end{equation*}
leaving all other reward laws and the transition kernel unchanged. It is easy to verify that these are probability densities, which are strictly positive on $(0,1)$ and polynomial in $r$, with uniform density and Lipschitz bounds.

Since every $t_{ij}\in(1/4,3/4)$, the shifted reward arguments remain in $(0,1)$ near $(\btheta_0,0)$, where $\bh_\beta(\btheta)$ is therefore continuously differentiable. Since $\bG_0$ is invertible, the implicit function theorem yields a continuously differentiable solution $\btheta(\beta)$ to $\bh_\beta(\btheta(\beta))=\bm0$ for sufficiently small $|\beta|$. Implicit differentiation gives
\[
\begin{aligned}
\psi'(0)
&=\bm c^\top\btheta'(0)\\
&=-\varphi(0)\sum_{i=1}^m
v_{s_0,i}\EB_0[Y_i(R)z(R)]\\
&=-\varphi(0)\EB_0[H(R)z(R)]\ne0.
\end{aligned}
\]
Choosing $b>0$ sufficiently small ensures that $\btheta(\beta)$ is continuously differentiable on a neighborhood of $J=[-b,b]$, all shifted reward arguments remain in $(1/4,3/4)$, and $\min_{\beta\in J}|\psi'(\beta)|>0$. This completes the proof of both assertions.
\end{proof}

%% file: Appendix_numerical.tex
\subsection{Construction of the Simglucose Model}
\label{app:simglucose_model}

We calibrate the model using $288$ seven-day trajectories of the Simglucose virtual patient \texttt{adult\#001}, simulated under the original basal--bolus controller with a target glucose level of $140$ mg/dL and random meal scenarios. The environment and controller interact at three-minute intervals, and we aggregate twenty consecutive steps into each one-hour transition. We exclude the first $24$ hours of each trajectory as warm-up and the final two hours from calibration.

The state representation combines clinical glucose ranges with physiological information. We partition the simulator's subcutaneous glucose level at $54$, $70$, $100$, $140$, $180$, and $250$ mg/dL. Within each range, K-means clustering uses seven features: the current glucose level, its rate of change over the preceding $30$ minutes, gastrointestinal carbohydrate content, subcutaneous insulin content, an internal insulin-action variable, and the sine and cosine of the time of day. We apply the transformation $\log(1+x)$ to the carbohydrate and subcutaneous insulin contents, and standardize all seven features using means and standard deviations estimated from the calibration records before clustering. Table~\ref{tab:simglucose_states} gives the number of clusters in each range. The range at or above $250$ mg/dL contains no calibration observations and therefore contributes no states, leaving $15$ states.

\begin{table}[tbp]
\centering
\begin{tabular}{lc}
\toprule
Glucose range (mg/dL) & Number of states \\
\midrule
$<54$       & 1 \\
$[54,70)$   & 2 \\
$[70,100)$  & 3 \\
$[100,140)$ & 4 \\
$[140,180)$ & 3 \\
$[180,250)$ & 2 \\
\midrule
Total & 15 \\
\bottomrule
\end{tabular}
\caption{Clinical glucose ranges and within-range cluster counts in the Simglucose model.}
\label{tab:simglucose_states}
\end{table}


We use the blood-glucose risk score of~\citet{Kovatchev1997}, which symmetrizes the risks associated with hypoglycemia and hyperglycemia. For a glucose observation $g$ in mg/dL, the score is
\[
\mathcal R(g)=
\begin{cases}
100, & g\le 20 \text{ or } g\ge 600,\\
10\left[1.509\bigl\{(\log g)^{1.084}-5.381\bigr\}\right]^2,
& 20<g<600.
\end{cases}
\]
For each one-hour interval, let $g_1,\ldots,g_{20}$ be the glucose values returned at the three-minute environment steps. We compute the average risk
\[
\overline{\mathcal R}
=\frac1{20}\sum_{j=1}^{20}\mathcal R(g_j)
\]
and transform it into the bounded utility
\[
U=\exp\left\{-\frac{\log 2}{10}\overline{\mathcal R}\right\}.
\]
Each calibration record includes the current state, successor state, and utility. Given a source state $s$, the model samples a recorded successor--utility pair $(v,u)$ from the empirical distribution of records with source state $s$. The reward is then sampled conditionally on $u$ according to
\[
R\mid U=u\sim
(1-\varepsilon)\operatorname{Beta}
\bigl(2+\kappa u,\,2+\kappa(1-u)\bigr)
+\varepsilon\operatorname{Unif}[0,1],
\qquad
\kappa=1000,\quad \varepsilon=10^{-4}.
\]
The calibrated joint reward--successor kernel satisfies Assumptions~\ref{assump:density}--\ref{assump:recurrence} and the conditions in Remark~\ref{rem:joint-reward-transition}. All three sampling mechanisms use this fixed kernel.


\subsection{Construction of the Electricity Settlement MRP}
\label{app:elia_model}
The application in Section~\ref{subsec:experiment_elia} uses Elia's ODS047 dataset of quarter-hourly imbalance settlement records. We use data from 2015 to determine the state boundaries and construct $175391$ consecutive transitions from the 2016--2020 records for inference.

Let $I_k$ and $V_k$ denote the system imbalance and net regulation volume in period $k$, both measured in MW. We distinguish shortage, approximate balance, and surplus using the sign and magnitude of $I_k$, and further classify shortage and surplus by $|V_k|$. Let $\widehat Q_p$ denote the empirical $p$-quantile computed from the 2015 records using linear interpolation. The thresholds are
\[
\begin{aligned}
b &\coloneqq \widehat Q_{0.20}(|I|)=32.8094,\\
a_1 &\coloneqq \widehat Q_{1/3}(|V|)=49.288,\\
a_2 &\coloneqq \widehat Q_{2/3}(|V|)=110.7,
\end{aligned}
\]
all in MW. Define the regulation magnitude category
\[
B(v)=
\begin{cases}
0, & |v|\le a_1,\\
1, & a_1<|v|\le a_2,\\
2, & |v|>a_2,
\end{cases}
\]
and assign each settlement period to the state
\[
Z_k=
\begin{cases}
B(V_k), & I_k<-b,\\
3, & |I_k|\le b,\\
4+B(V_k), & I_k>b.
\end{cases}
\]
Thus, $S0,S1,S2$ represent shortage with increasing regulation magnitude, $S3$ represents approximate balance, and $S4,S5,S6$ represent surplus with increasing regulation magnitude. The thresholds remain fixed throughout the evaluation period.

Let $P_k^+$ denote the signed settlement price in the dataset's ``Positive imbalance price'' field, measured in EUR/MWh. We transform prices into bounded rewards using
\[
g(p)=\frac12+\frac1\pi\arctan\left(\frac{p}{100}\right).
\]
For consecutive settlement periods $k$ and $k+1$, we construct the observed transition as
\[
(S_t,R_t,S_{t+1})
=\bigl(Z_k,g(P_{k+1}^+),Z_{k+1}\bigr).
\]
Both inference procedures use the same sequence of observed transitions.

\subsection{Implementation of Algorithm~\ref{alg:dyadic_histogram}}\label{app:dh_implementation}
The comparisons in~\eqref{eq:dh_score_lower}--\eqref{eq:dh_score_upper} make $\underline E_t$ piecewise constant on $\Theta_m$. Instead of enumerating all the resulting regions, we compute an outer approximation by testing coordinate slices of the form $\theta(s,i)\le x$ or $\theta(s,i)\ge x$. Within each slice, the coordinate ranges provide lower bounds on $g^-$ and upper bounds on $g^+$ in every occupied bin. These yield a lower bound on $\underline E_t$ throughout the slice via the positive and negative betting products. If this bound exceeds $B_\alpha$, the entire slice is excluded. The remaining coordinate bounds determine the reported outer interval $[L_{\rm out},U_{\rm out}]$ for $\bm c^\top\btheta$.

To assess numerical accuracy, we also evaluate $\underline E_t$ at candidate vectors and let $L_{\rm in}$ and $U_{\rm in}$ be the smallest and largest contrast values among those satisfying $\underline E_t<B_\alpha$. If $[L_*,U_*]$ denotes the exact Dyadic interval, then $L_{\rm out}\le L_*\le L_{\rm in}$ and $U_{\rm in}\le U_*\le U_{\rm out}$. Consequently,
\[
0\le (U_{\rm out}-L_{\rm out})-(U_*-L_*)\le G,
\qquad
G\coloneqq (L_{\rm in}-L_{\rm out})+(U_{\rm out}-U_{\rm in}).
\]
Thus, $G$ is a certified upper bound on the excess width of the reported interval. At $T=500000$, the mean values of $G$ in RiverSwim are $0.004174$, $0.009385$, and $0.010545$ under synchronous, asynchronous, and Markovian sampling, respectively. The corresponding means of $G/(U_{\rm out}-L_{\rm out})$ are $18.81\%$, $19.15\%$, and $18.93\%$. In Simglucose, both endpoint gaps are at most $10^{-3}$ in all $1000$ synchronous and asynchronous replications and in $991$ of the $1000$ Markovian replications.